\documentclass[table]{article}

\usepackage{graphicx}

\usepackage[round]{natbib}  
\usepackage{fullpage}
\usepackage{authblk}

\usepackage[utf8]{inputenc} 
\usepackage[T1]{fontenc}    
\usepackage{hyperref}       
\usepackage{url}            
\usepackage{booktabs}       
\usepackage{amsfonts}       
\usepackage{nicefrac}       
\usepackage{microtype}      

\usepackage{amsmath}
\usepackage{amsthm}
\usepackage{algorithm}
\usepackage{algorithmic}
\usepackage{xspace}
\usepackage[dvipsnames]{xcolor}
\usepackage{graphicx}
\usepackage{caption}
\usepackage{subcaption}
\usepackage{multirow}
\usepackage{tablefootnote}
\usepackage[usestackEOL]{stackengine}
\usepackage{colortbl} 
\usepackage{hhline}

\usepackage{enumitem}
\usepackage{Definitions}

\newcommand{\comment}[1]{}

\def\E{\mathbb{E}}

\def\qpi{Q^{\pi}}

\def\mdp{\mathcal{M}}
\def\init{\mu_0}

\def\visitpi{d^\pi}

\def\visitrb{d^\Dset}

\DeclareMathOperator{\avgstep}{\rho}

\def\bellman{\mathcal{B}}

\def\Sset{S}
\def\Aset{A}

\def\Dset{\mathcal{D}}

\def\defeq{:=}
\renewcommand{\widehat}{\hat}

\def\rew{R}
\def\bellman{\gamma\cdot \mathcal{P}^{\pi}}
\def\bellmant{\gamma\cdot \mathcal{P}^{\pi}_{*}}
\def\bellmannog{\mathcal{P}^{\pi}}
\def\bellmantnog{\mathcal{P}^{\pi}_{*}}

\def\qvar{Q}
\def\dvar{d}

\def\initsamp{\substack{a_0\sim\pi(s_0) \\ s_0\sim\init}}

\def\sampone{\substack{(s,a)\sim\visitrb}}
\def\samptwo{\substack{(s,a,r,s')\sim\visitrb \\ a'\sim\pi(s')}}

\def\ie{\emph{i.e.}\xspace}
\def\eg{\emph{e.g.}\xspace}
\def\dlp{$d$-LP\xspace}
\def\qlp{$Q$-LP\xspace}

\renewcommand{\cite}{\citep}

\hypersetup{
colorlinks=true,
}
\newcommand{\positive}{{\color{Turquoise}\ge 0}}
\newcommand{\norm}{{\color{green}\lambda}}
\newcommand{\reward}{{\color{Plum}\alpha_R}}
\newcommand{\preg}{{\color{magenta}\alpha_Q}}
\newcommand{\dreg}{{\color{Orange}\alpha_\zeta}}

\title{Off-Policy Evaluation via the Regularized Lagrangian}

\author{
  $^*$Mengjiao Yang$^1$, \thanks{indicates equal contribution. Email: \texttt{\{sherryy, ofirnachum, bodai\}@google.com}.}
  Ofir Nachum$^1$, $^*$Bo Dai$^1$\\\vspace{-2mm}
  Lihong Li$^1$, Dale Schuurmans$^{1,2}$\\ \vspace{3mm}
  $^1$Google Research, Brain Team \quad $^2$University of Alberta
}

\begin{document}

\date{}
\maketitle


\begin{abstract}
  The recently proposed \emph{distribution correction estimation} (DICE) family of estimators has advanced the state of the art in off-policy evaluation from behavior-agnostic data. While these estimators all perform some form of stationary distribution correction, they arise from different derivations and objective functions. In this paper, we unify these estimators as regularized Lagrangians of the same linear program. The unification allows us to expand the space of DICE estimators to new alternatives that demonstrate improved performance. More importantly, by analyzing the expanded space of estimators both mathematically and empirically we find that dual solutions offer greater flexibility in navigating the tradeoff between optimization stability and estimation bias, and generally provide superior estimates in practice.
\end{abstract}

\section{Introduction}\label{sec:introduction}
One of the most fundamental problems in reinforcement learning (RL) is
\emph{policy evaluation},
where we seek to estimate 
the expected long-term payoff of a given
\emph{target} policy 
in a decision making environment.
An important variant of this problem,
\emph{off-policy evaluation} (OPE)~\cite{Precup00ET},
is motivated by applications
where deploying a policy in a live environment
entails significant cost or risk~\cite{Murphy01MM,Thomas15HCPE}.
To circumvent these issues,
OPE attempts to estimate the value of a target policy
by referring only to a dataset of experience previously gathered
by other policies in the environment.
Often, such logging or \emph{behavior} policies are not known explicitly
(\eg, the experience may come from human actors),
which necessitates the use of \emph{behavior-agnostic} OPE
methods~\cite{NacChoDaiLi19}.

While behavior-agnostic OPE appears to be a daunting problem,
a number of estimators have recently been developed for this scenario
\cite{NacChoDaiLi19,uehara2019minimax,zhang2020gendice,zhang2020gradientdice},
demonstrating impressive empirical results.
Such estimators, known collectively as the ``DICE'' family
for \emph{DIstribution Correction Estimation},
model the ratio between the propensity 
of the target policy to visit particular state-action pairs
relative to their likelihood of appearing in the logged data.
A distribution corrector of this form can then be directly used
to estimate the value of the target policy.

Although there are many commonalities between the various DICE estimators,
their derivations are distinct and seemingly incompatible.
For example, \emph{DualDICE}~\cite{NacChoDaiLi19}
is derived by a particular change-of-variables technique, 
whereas \emph{GenDICE}~\cite{zhang2020gendice}
observes that the substitution strategy cannot work in the average reward
setting,
and proposes a distinct derivation based on distribution matching.
\emph{GradientDICE}~\cite{zhang2020gradientdice} notes
that GenDICE exacerbates optimization difficulties, 
and proposes a variant designed for limited sampling capabilities.
Despite these apparent differences in these methods,
the algorithms all involve a minimax optimization that has
a strikingly similar form,
which suggests that there is a common connection 
that underlies
the alternative derivations.

We show that the previous DICE formulations are all in fact
equivalent to regularized Lagrangians of the same linear program (LP).
This LP shares an intimate relationship with the policy evaluation problem,
and has a primal form we refer to as the \qlp and a dual form we refer
to as the \dlp.
The primal form has been concurrently identified and
studied in the context of policy optimization~\cite{algae},
but we focus on the \dlp formulation for off-policy evaluation here,
which we find to have a more succinct and revealing form for this purpose.
Using the \dlp, we identify a number of key choices in translating it into
a \emph{stable} minimax optimization problem -- 
\ie whether to include redundant constraints,
whether to regularize the primal or dual variables
-- in addition to choices in how to translate an optimized solution into
an \emph{asymptotic unbiased}, ``unbiased'' for short, estimate of the policy value.
We use this characterization to show that the known members of the DICE family
are a small subset of specific choices made within a much larger,
unexplored set of potential OPE methods.

To understand the consequences of the various choices,
we provide a comprehensive study.
First, we theoretically investigate which configurations
lead to bias in the primal or dual solutions,
and when this affects the final estimates.
Our analysis shows that the dual solutions
offer greater flexibility in stabilizing the optimization
while preserving asymptotic unbiasedness,
versus primal solutions. We also perform an extensive empirical evaluation of the various choices
across different domains and function approximators,
and identify novel configurations that improve 
the observed
outcomes.

\section{Background}\label{sec:background}
We consider an infinite-horizon Markov Decision Process
(MDP)~\citep{puterman1994markov}, specified by a tuple
$\mdp = \langle \Sset, \Aset, \rew, T, \init, \gamma \rangle$, 
which consists
of a state space, action space, reward function,
transition probability function, initial state distribution,
and discount factor $\gamma\in [0, 1]$.%
\footnote{
For simplicity, we focus on the discounted case where $\gamma\in[0,1)$ 
unless otherwise specified.
The same conclusions generally hold for the undiscounted case with $\gamma = 1$;
see \appref{appendix:undiscounted} for more details.
} 
A policy $\pi$ interacts with the environment
starting at an initial state $s_0 \sim \init$,
producing 
a distribution $\pi(\cdot|s_t)$ over $\Aset$
from which an action $a_t$ is sampled
and applied to the environment at each step $t \ge 0$.
The environment produces a scalar reward $r_t=\rew(s_t, a_t)$,\footnote{
We consider a a deterministic reward function. All of our results apply to stochastic rewards as well.
} 
and transitions to a new state $s_{t+1} \sim T(s_t, a_t)$.

\subsection{Policy Evaluation}
The \emph{value} of a policy $\pi$ is defined as the normalized expected
per-step reward it obtains:
\begin{equation}
    \displaystyle \avgstep(\pi) \defeq (1-\gamma) \E\left[\left.\textstyle\sum_{t=0}^\infty \gamma^t \rew(s_t,a_t) ~\right| s_0\sim\init, \forall t, a_t\sim \pi(s_t), s_{t+1}\sim T(s_t, a_t)\right].
    \label{eqn:avgstep}
\end{equation}
In the policy evaluation setting,
the policy being evaluated is referred to as the \emph{target} policy.
The value of a policy may be expressed in two equivalent ways:
\begin{equation}
  \avgstep(\pi) = (1-\gamma) \cdot \E_{\initsamp}[\qpi(s_0,a_0)] = \E_{(s,a)\sim\visitpi}[R(s,a)],
\end{equation}
where $\qpi$ and $\visitpi$ are the \emph{state-action values} and 
\emph{visitations}
of $\pi$, respectively, 
which satisfy
{\small
\begin{equation}\label{eq:bellman-q}
  \textstyle\qpi(s,a) = \rew(s,a) + \bellman \qpi(s,a),
    \text{ where }  \bellmannog \qvar(s,a) \defeq \E_{s'\sim T(s,a),a'\sim\pi(s')}[\qvar(s',a')]~, 
\end{equation}}
{\small
\begin{equation}\label{eq:bellman-d}
  \textstyle\visitpi(s,a) = (1-\gamma)\init(s)\pi(a|s) + \bellmant\visitpi(s,a),
    \text{ where }  \bellmantnog \dvar(s,a) \defeq \pi(a|s) \sum_{\tilde{s},\tilde{a}} T(s|\tilde{s},\tilde{a})\dvar(\tilde{s},\tilde{a}).   \
\end{equation}}

Note that $\bellmannog$ and $\bellmantnog$ are linear operators
that are transposes (adjoints) of each other. 
We refer to $\bellmannog$ as the \emph{policy transition operator} and 
$\bellmantnog$ as the \emph{transpose policy transition operator}. 
The function $\qpi$ corresponds to the $Q$-values of the policy $\pi$;
it maps state-action pairs $(s,a)$ to the expected value of policy $\pi$
when run in the environment starting at $(s,a)$.
The function $\visitpi$ corresponds to the on-policy distribution of $\pi$;
it is the normalized distribution over state-action pairs $(s,a)$
measuring the likelihood 
$\pi$ enounters
the pair $(s,a)$,
averaging over time via $\gamma$-discounting.
We make the
following standard assumption, which is common in previous policy evaluation
work~\citep{zhang2020gendice,algae}.
\begin{assumption}[MDP ergodicity]\label{asmp:mdp_reg}
There is unique fixed point solution to~\eqref{eq:bellman-d}. 
\end{assumption}
When $\gamma \in [0, 1)$, \eqref{eq:bellman-d} always has a unique solution,
as $0$ cannot belong to the spectrum of
$I - \gamma \bellmantnog$.
For $\gamma\!=\!1$,
the assumption reduces to ergodicity for discrete case under a restriction of $d$ to a normalized distribution;
the continuous case is treated by~\citet{meyn2012markov}. 

\subsection{Off-policy Evaluation via the DICE Family}
Off-policy evaluation (OPE) aims to estimate $\avgstep(\pi)$ using only a
\emph{fixed} dataset of experiences. 
Specifically, we assume access to a finite dataset
$\Dset=\{(s_0^{(i)}, s^{(i)},a^{(i)},r^{(i)},s^{\prime(i)})\}_{i=1}^N$,
where $s_0^{(i)}\sim\init$, $(s^{(i)},a^{(i)})\sim\visitrb$
are samples from some unknown distribution $\visitrb$,
$r^{(i)}=\rew(s^{(i)},a^{(i)})$,
and $s^{\prime(i)}\sim T(s^{(i)},a^{(i)})$.
We at times abuse notation and use $(s,a,r,s')\sim\visitrb$ or
$(s,a,r)\sim\visitrb$ as a shorthand for $(s,a)\sim\visitrb, r=R(s,a),
s'\sim T(s, a)$,
which simulates sampling from the dataset $\Dset$ when using a
finite number of samples. 

The recent DICE methods take advantage of the following expression
for the policy value:
\begin{equation}
    \label{eq:avgstep-ratio}
    \textstyle
    \avgstep(\pi) = \E_{(s,a,r)\sim\visitrb}\left[\zeta^*(s,a)   \cdot r \right],
\end{equation}
where $\zeta^*\rbr{s, a}\defeq\visitpi(s,a)/\visitrb(s, a)$
is the \emph{distribution correction ratio}.
The existing DICE estimators seek to approximate this ratio 
without knowledge of $\visitpi$ or $\visitrb$, 
and then apply~\eqref{eq:avgstep-ratio}
to derive an estimate of $\avgstep(\pi)$.
This general paradigm is supported by the following assumption.
\begin{assumption}[Boundedness]\label{asmp:bounded_ratio}
The stationary correction ratio is bounded, $\nbr{\zeta^*}_\infty\le C<\infty$.
\end{assumption}
When $\gamma < 1$,
DualDICE~\citep{NacChoDaiLi19} chooses a convex objective
whose optimal solution corresponds to this ratio,
and employs a change of variables
to transform the dependence
on $\visitpi$ to
$\init$.
GenDICE~\citep{zhang2020gendice}, on the other hand,
minimizes a divergence between
successive on-policy state-action distributions,
and introduces a normalization constraint
to ensure the estimated ratios average to $1$ over the off-policy dataset.
Both DualDICE and GenDICE apply Fenchel duality to reduce an intractable
convex objective to a minimax objective,
which enables sampling and optimization in a stochastic or
continuous action space. 
GradientDICE~\citep{zhang2020gradientdice} extends GenDICE
by using a linear parametrization so that the minimax optimization
is convex-concave with convergence guarantees. 

\section{A Unified Framework of DICE Estimators}\label{sec:unification}

In this section, given a fixed target policy $\pi$, we present a linear programming representation (LP) of its state-action stationary distribution $d^\pi\rbr{s, a}\in \Pcal$, referred to as the $d$-LP. 
The dual of this LP has solution $\qpi$, thus revealing the duality between the $Q$-function and the $d$-distribution of any policy $\pi$. 
We then discuss the mechanisms by which one can improve optimization stability through the application of regularization and redundant constraints. 
Although in general this may introduce bias into the final value estimate, there are a number of valid configurations for which the resulting estimator for $\rho(\pi)$ remains \emph{unbiased}.
We show that existing DICE algorithms cover several choices of these configurations, while there is also a sizable subset which remains unexplored.

\subsection{Linear Programming Representation for the $d^\pi$-distribution}\label{sec:dlp}
The following theorem presents a formulation of $\rho(\pi)$ in terms of a linear program with respect to the constraints in~\eqref{eq:bellman-d} and~\eqref{eq:bellman-q}.
\begin{theorem}\label{thm:dual-succinct}
Given a policy $\pi$, under~\asmpref{asmp:mdp_reg}, its value $\rho\rbr{\pi}$ defined in~\eqref{eqn:avgstep} can be expressed by the following $d$-LP:
\begin{equation}\label{eq:dual-succinct}
\max_{d:S\times A\rightarrow \RR}\,\, \EE_{d}\sbr{\rew\rbr{s, a}},\quad \st,\quad d(s,a) = \underbrace{(1-\gamma)\init(s)\pi(a|s) + \bellmant d(s,a)}_{\Bcal_*^\pi\cdot d}.
\end{equation}
  We refer to the $d$-LP above as the \textbf{dual} problem. Its corresponding \textbf{primal} LP is
\begin{equation}\label{eq:primal-succinct}
\min_{Q:S\times A\rightarrow \RR}\,\, \rbr{1 - \gamma}\EE_{\init\pi}\sbr{Q\rbr{s, a}},\quad \st,\quad   Q(s,a) = \underbrace{\rew(s,a) + \bellman Q(s,a)}_{\Bcal^\pi \cdot Q}.
\end{equation}
\end{theorem}
\begin{proof}
Notice that under~\asmpref{asmp:mdp_reg}, the constraint in~\eqref{eq:dual-succinct} determines a unique solution, which is the stationary distribution $d^\pi$. Therefore, the objective will be $\rho\rbr{\pi}$ by definition. On the other hand, due to the contraction of $\gamma\cdot \bellmannog$, the primal problem is feasible and the solution is $Q^\pi$, which shows the optimal objective value will also be $\rho\rbr{\pi}$, implying strong duality holds.
\end{proof}

\thmref{thm:dual-succinct} presents a succinct LP representation for policy value and reveals the duality between the $Q^\pi$-function and $d^\pi$-distribution, thus providing an answer to the question raised by~\citet{uehara2019minimax}. 
Although the $d$-LP provides a mechanism for policy evaluation, directly solving either the primal or dual $d$-LPs is difficult due to the number of constraints, which will present difficulties when the state and action spaces is uncountable. 
These issues are exaggerated in the off-policy setting where one only has access to samples $(s_0,s,a,r,s')$ from a stochastic process.
To overcome these difficulties, one can instead approach these primal and dual LPs through the Lagrangian, 
\begin{equation*}
\resizebox{0.98\textwidth}{!}
{
$
  \max_{d} \min_{\qvar} L(d, \qvar) \defeq (1-\gamma)\cdot\E_{\initsamp}[\qvar(s_0,a_0)] + \sum_{s,a}\dvar(s,a)\cdot(R(s,a) + \gamma\bellmannog\qvar(s,a) - \qvar(s,a)).
$
}
\end{equation*}
In order to enable the use of an arbitrary off-policy distribution $\visitrb$, we make the change of variables $\zeta(s,a)\defeq d(s,a)/\visitrb(s,a)$. This yields an equivalent Lagrangian in a more convenient form:
\begin{multline}  \label{eq:lagrangian-change-var}
  \max_{\zeta} \min_{\qvar} L_D(\zeta, \qvar) \defeq  (1-\gamma)\cdot\E_{\initsamp}[\qvar(s_0,a_0)]
 + \E_{\samptwo}[\zeta(s,a)\cdot(r + \gamma \qvar(s',a') - \qvar(s,a) )].
\end{multline}
The Lagrangian has primal and dual solutions $Q^* = Q^\pi$ and $\zeta^* = d^\pi / \visitrb$. 
Approximate solutions 
to one or both of 
$\hat{Q},\hat{\zeta}$ 
can be used to
estimate $\hat{\rho}(\pi)$, by either using the standard DICE paradigm in~\eqref{eq:avgstep-ratio} which corresponds to the dual objective in~\eqref{eq:dual-succinct} or, alternatively, by using the primal objective in~\eqref{eq:primal-succinct} or the Lagrangian objective in~\eqref{eq:lagrangian-change-var}; we further discuss these choices later in this section.
Although the Lagrangian in~\eqref{eq:lagrangian-change-var} should in principle be able to derive the solutions $\qpi, \visitpi$ and so yield accurate estimates of $\avgstep(\pi)$, in practice there are a number of optimization difficulties that are liable to be encountered.
Specifically, even in tabular case, due to lack of curvature, the Lagrangian is not strongly-convex-strongly-concave, and so one cannot guarantee the convergence of the final solution with stochastic gradient descent-ascent~(SGDA). 
These optimization issues can become more severe when moving to the continuous case with neural network parametrization, which is the dominant application case in practice.  
In order to mitigate these issues, we present a number of ways to introduce more stability into the optimization and discuss how these mechanisms may trade-off with the bias of the final estimate.
We will show that the application of certain mechanisms recovers the existing members of the DICE family, while a larger set remains unexplored.

\subsection{Regularizations and Redundant Constraints}\label{sec:reg_lagrangian}
The augmented Lagrangian method~(ALM)~\citep{rockafellar1974augmented} is proposed exactly for circumventing the optimization instability, where strong convexity is introduced by adding extra regularizations \emph{without} changing the optimal solution. However, directly applying ALM, \ie, adding $h_p\rbr{Q}\defeq \nbr{\Bcal^\pi \cdot\qvar - \qvar}^2_{d^\Dcal}$ or $h_d\rbr{d}\defeq D_f\rbr{d||\Bcal_*^\pi \cdot d}$ where $D_f$ denotes the $f$-divergence, will introduce extra difficulty, both statistically and algorithmically, due to the conditional expectation operator in $\Bcal^\pi$ and $\Bcal_*^\pi$ inside of the non-linear function in $h_p\rbr{Q}$ and $h_d\rbr{d}$, which is known as ``double sample'' in the RL literature~\citep{Baird95residualalgorithms}. Therefore, the vanilla stochastic gradient descent is no longer applicable~\citep{dai2016learning}, due to the bias in the gradient estimator. 

In this section, we use the spirit of ALM but explore other choices of regularizations to introduce strong convexity to the original Lagrangian~\eqref{eq:lagrangian-change-var}. 
In addition to regularizations, we also employ the use of redundant constraints, which serve to add more structure to the optimization without affecting the optimal solutions.
We will later analyze for which configurations these modifications of the original problem will lead to biased estimates for $\avgstep(\pi)$.

We first present the unified objective in full form equipped with all 
choices of 
regularizations and redundant 
constraints:
\begin{align}
\label{eq:lagrangian_reg}
\nonumber \max_{\zeta\positive} \min_{\qvar, \norm} L_D(\zeta, \qvar, \norm) \defeq  & (1-\gamma)\cdot\E_{\initsamp}[\qvar(s_0,a_0)] + \norm \\
\nonumber &+ \E_{\samptwo}[\zeta(s,a)\cdot(\reward\cdot R(s,a) + \gamma \qvar(s',a') - \qvar(s,a) - \norm)] \\
  &+ \preg\cdot\E_{\sampone}[f_1(\qvar(s,a))] - \dreg\cdot \E_{\sampone}[f_2(\zeta(s,a))].
\end{align}

Now, let us explain each term in $\rbr{\preg, \dreg, \reward, \zeta\positive, \norm}$.
\begin{itemize}
  \item {\bf {\color{magenta} Primal} and {\color{Orange}Dual} regularization:} To introduce better curvature into the Lagrangian, we introduce primal and dual regularization $\preg\EE_{\visitrb}\sbr{f_1\rbr{\qvar}}$ or $\dreg\EE_{\visitrb}\sbr{f_2\rbr{\zeta}}$, respectively. Here $f_1,f_2$ are some convex and lower-semicontinuous functions.
  \item {\bf\color{Plum} Reward:} Scaling the reward may be seen as an extension of the dual regularizer, as it is a component in the dual objective in~\eqref{eq:dual-succinct}. 
    We consider $\reward\in\{0, 1\}$.

  \item {\bf \color{Turquoise} Positivity:}
Recall that the solution to the original Lagrangian is
$\zeta^*\rbr{s, a} = \frac{d^\pi\rbr{s, a}}{\visitrb\rbr{s, a}}\ge 0$.
We thus consider adding a positivity constraint to the dual variable.
    This may be interpreted as modifying the original $d$-LP in~\eqref{eq:dual-succinct} to add a condition $d\ge0$ to its objective.
  \item {\bf \color{green} Normalization:} Similarly, the normalization constraint also comes from the property of the optimal solution $\zeta^*\rbr{s, a}$, \ie, $\EE_{\visitrb}\sbr{\zeta\rbr{s, a}} = 1$. If we add an extra constraint to the $d$-LP~\eqref{eq:dual-succinct} as $\sum_{s,a}\dvar(s,a) = 1$ and apply the Lagrangian, we result in the term $\norm - \EE_{\visitrb}\sbr{\norm \zeta\rbr{s, a}}$ seen in~\eqref{eq:lagrangian_reg}.   
\end{itemize}
As we can see, the latter two options come from the properties of optimal dual solution, and this suggests that their inclusion would not affect the optimal dual solution. 
On the other hand, the first two options (primal/dual regularization and reward scaling) will in general affect the solutions to the optimization. 
Whether a bias in the solution affects the final estimate depends on the estimator being used.

\paragraph{Remark (Robust optimization justification):} Besides the motivation from ALM for strong convexity, the regularization terms in~\eqref{eq:lagrangian_reg}, $\preg\cdot\E_{\sampone}[f_1(\qvar(s,a))]$ and $\dreg\cdot \E_{\sampone}[f_2(\zeta(s,a))]$, can also be interpreted as introducing robustness with some perturbations to the Bellman differences. We consider the dual regularization as an example. Particularly, the Fenchel dual form $\dreg\cdot \E_{\sampone}[f_2(\zeta(s,a))] = \dreg \cbr{\max_{\delta\rbr{s, a}\in \Omega}\,\, \inner{\zeta}{\delta} - \E_{\sampone}\sbr{f_2^*\rbr{\delta\rbr{s, a}}}}$ where $\Omega$ denotes the domain of function $f^*_2$. For simplicity, we consider $f_2^*\rbr{\cdot} = \rbr{\cdot}^2$. Plug this back into~\eqref{eq:lagrangian_reg}, we obtain 
\begin{align}
\label{eq:ro_lagrangian_dreg}
\nonumber \max_{\zeta\positive} \min_{\qvar, \norm, \delta\in \Omega} L_D(\zeta, \qvar, \norm) \defeq  & (1-\gamma)\cdot\E_{\initsamp}[\qvar(s_0,a_0)] + \norm \\
\nonumber &+ \E_{\samptwo}[\zeta(s,a)\cdot(\reward\cdot R(s,a) + \gamma \qvar(s',a') - \qvar(s,a) - \norm - \dreg\delta\rbr{s, a})] \\
  &+ \preg\cdot\E_{\sampone}[f_1(\qvar(s,a))] + \dreg\cdot \EE_{\sampone}[{\delta^2\rbr{s, a}}],
\end{align}
which can be understood as introducing slack variables or perturbations in $L_2$-ball to the Bellman difference $\reward\cdot R(s,a) + \gamma \qvar(s',a') - \qvar(s,a)$. For different regularization, the perturbations will be in different dual spaces. From this perspective, besides the stability consideration, the dual regularization will also mitigate both statistical error, due to sampling effect in approximating the Bellman difference, and approximation error induced by parametrization of $Q$. Similarly, the primal regularization can be interpreted as introducing slack variables to the stationary state-action distribution condition, please refer to~\appref{appendix:primal_robust}. 

Given estimates $\hat{Q},\hat{\lambda},\hat{\zeta}$, there are three potential ways to estimate $\avgstep(\pi)$.
\begin{itemize}
  \item {\bf Primal estimator:} $\hat{\rho}_Q(\pi) \defeq (1-\gamma)\cdot \E_{\initsamp}[\hat\qvar(s_0,a_0)] + \hat\lambda.$
  \item {\bf Dual estimator:} $ \hat{\rho}_\zeta(\pi) \defeq \E_{(s,a,r)\sim\visitrb}[\hat\zeta(s,a)\cdot r].$
  \item {\bf Lagrangian:} 
    $
      \hat{\rho}_{Q,\zeta}(\pi) \defeq \hat{\rho}_Q(\pi) + \hat{\rho}_\zeta(\pi) + \E_{\samptwo}\sbr{\hat\zeta\rbr{s, a}(\gamma\hat\qvar(s',a') - \hat\qvar(s,a) - \hat\lambda)}.
    $
\end{itemize}

The following theorem outlines when a choice of regularizations, redundant constraints, and final estimator will provably result in an unbiased estimate of policy value.
\begin{theorem}[Regularization profiling]\label{thm:reg_profile}
Under~\asmpref{asmp:mdp_reg} and~\ref{asmp:bounded_ratio}, we summarize the effects of $(\preg$, $\dreg$, $\reward$, $\zeta\positive$, $\norm)$, which corresponds to {\color{magenta} primal} and {\color{Orange} dual} regularizations, w/w.o. {\color{Plum} reward}, and {\color{Turquoise} positivity} and {\color{green} normalization} constraints.
without considering function approximation. 
\begin{table*}[h]
\renewcommand{\arraystretch}{1.2}
\setlength{\tabcolsep}{3pt}
\centering
\begin{small}
  \begin{tabular}{c|c|c||c|c|c}
    \toprule        
\multicolumn{3}{l||}{ Regularization (with or without $\norm$) } & $\hat\rho_Q$ & $\hat\rho_\zeta$ & $\hat\rho_{Q,\zeta}$ \\\hline\hline
& \multirow{2}{*}{ $\reward=1$ } & $\zeta$ free & \textbf{Unbiased} & \multirow{2}{*}{ Biased } & {\textbf{Unbiased}} \\\cline{3-4}\cline{6-6}
$\dreg$ = 0 & & $\zeta\positive$ & \multirow{7}{*}{Biased} & & {Biased} \\\cline{2-3}\cline{5-6}
$\preg$ > 0 & \multirow{2}{*}{ $\reward=0$ } & $\zeta$ free &  & \multirow{6}{*}{ \textbf{Unbiased} } & \multirow{6}{*}{\textbf{Unbiased}} \\\cline{3-3}
& & $\zeta\positive$ & & \\\cline{1-3}
& \multirow{2}{*}{ $\reward=1$ } & $\zeta$ free &  & \\\cline{3-3}
$\dreg$ > 0 & & $\zeta\positive$ & & \\\cline{2-3}
$\preg$ = 0 & \multirow{2}{*}{ $\reward=0$ } & $\zeta$ free & & \\\cline{3-3}
& & $\zeta\positive$ & & \\
\bottomrule
\end{tabular}
\end{small}
\end{table*}
\end{theorem}
Notice that the primal and dual solutions can both be unbiased under
specific regularization configurations, 
but the dual solutions are unbiased in 6 out of 8 such configurations,
whereas the primal solution is unbiased in only 1 configuration.
The primal solution additionally requires the positivity constraint to be
excluded (see details in~\appref{appendix:reg_prof}), further restricting
its optimization choices.

The Lagrangian estimator is unbiased when at least one of $\hat Q, \hat\lambda$ or $\hat\zeta$ are unbiased.
This property is referred to as \emph{doubly robust} in the literature~\citep{jiang2015doubly}
This seems to imply that the Lagrangian estimator is optimal for behavior-agnostic off-policy evaluation.
However, this is not the case as we will see in the empirical analysis. 
Instead, the approximate dual solutions are typically more accurate than
approximate primal solutions. Since neither is exact, the Lagrangian suffers from error in both, while the dual estimator $\hat{\rho}_\zeta$ will exhibit more robust performance, as it solely relies on the approximate $\hat\zeta$.

\subsection{Recovering Existing OPE Estimators}\label{sec:dice_lagrangian}
This organization provides
a complete picture of the DICE family of estimators.
Existing DICE estimators can simply be recovered
by picking one of the valid regularization configurations:
\begin{itemize}
\item{\textbf{DualDICE}~\cite{NacChoDaiLi19}: $\rbr{\preg=0, \dreg=1, \reward=0}$ without $\zeta\positive$ and without $\norm$. DualDICE also derives an \emph{unconstrained primal form} where optimization is exclusively over the primal variables (see~\appref{appendix:alter_primal}). This form results in a~\emph{biased} estimate but avoids difficults in minimax optimization, which again is a tradeoff between optimization stability and solution unbiasedness.}

\item{\textbf{GenDICE}~\cite{zhang2020gendice} and \textbf{GradientDICE}~\cite{zhang2020gradientdice}: $\rbr{\preg=1, \dreg=0, \reward=0}$ with $\norm$. GenDICE differs from GradientDICE in that GenDICE enables $\zeta\positive$ whereas GradientDICE disables it.

\item{\textbf{DR-MWQL} and \textbf{MWL}}~\cite{uehara2019minimax}:$(\preg = 0, \dreg = 0, \reward = 1)$ and $(\preg=0, \dreg=0, \reward = 0)$, both without $\zeta\positive$ and without $\norm$.}

\item{\textbf{LSTDQ}~\citep{lagoudakis2003least}}: With linear parametrization for $\tau(s, a)=\alpha^\top \phi\rbr{s, a}$ and $\qvar\rbr{s, a} = \beta^\top \phi\rbr{s, a}$, for any \emph{unbiased} estimator without $\xi\positive$ and $\norm$ in~\thmref{thm:reg_profile}, we can recover LSTDQ. Please refer to~\appref{appendix:recover_ope} for details.

\item{\textbf{Algae $Q$-LP}~\citep{algae}: $\rbr{\preg=0, \dreg=1, \reward = 1, \zeta\positive}$ without $\zeta\positive$ and without $\norm$. }

\item{\textbf{BestDICE:} $(\preg = 0, \dreg = 1,\reward = 0/1)$ with $\zeta\positive$ and with $\norm$.
More importantly, we discover a variant that achieves the best performance, which was not identified without this unified framework.}
\end{itemize}

\section{Experiments}\label{sec:experiment}
In this section, we empirically verify the theoretical findings.
We evaluate different choices of estimators, regularizers, and constraints,
on a set of OPE tasks ranging from tabular (Grid) to discrete-control
(Cartpole) and continuous-control (Reacher),
under linear and neural network parametrizations,
with offline data collected from behavior policies with different noise levels
($\pi_1$ and $\pi_2$).
See \appref{appendix:exp} for implementation details and additional results.
Our empirical conclusions are as follows:
\begin{itemize}
    \item The dual estimator $\hat\rho_\zeta$ is unbiased under more configurations and yields best performance out of all estimators, and furthermore exhibits strong robustness to scaling and shifting of MDP rewards.
    \item Dual regularization ($\alpha_\zeta>0$) yields better estimates than primal regularization; the choice of $\alpha_R\in\{0,1\}$ exhibits a slight advantage to $\alpha_R=1$.
    \item The inclusion of redundant constraints ($\lambda$ and $\zeta\ge0$) improves
stability and estimation performance.
    \item As expected, optimization using the unconstrained primal form is more stable but also more biased than optimization using the minimax regularized Lagrangian.
\end{itemize}
Based on these findings, we propose a particular set of choices that
generally performs well,
overlooked by previously proposed DICE estimators:
the dual estimator $\hat\rho_\zeta$ with regularized dual variable
($\alpha_\zeta>0,\alpha_R=1$) and redundant constraints
($\lambda, \zeta\ge0$) optimized with the Lagrangian.

\subsection{Choice of Estimator ($\hat\rho_Q$, $\hat\rho_\zeta$, or $\hat\rho_{Q,\zeta}$)}
\label{sec:exp_est}
We first consider the choice of estimator. In each case, we perform Lagrangian optimization with regularization chosen according to~\thmref{thm:reg_profile} to not bias the resulting estimator. We also use $\alpha_R=1$ and include redundant constraints for $\lambda$ and $\zeta\ge0$ in the dual estimator. Although not shown, we also evaluated combinations of regularizations which can bias the estimator (as well as no regularizations) and found that these generally performed worse; see~\secref{sec:exp_reg} for a subset of these experiments.

Our evaluation of different estimators is presented in~\figref{fig:est_reg}.
We find that the dual estimator consistently produces the best estimates across different tasks and behavior policies. In comparison, the primal estimates are significantly worse. While the Lagrangian estimator can improve on the primal, it generally exhibits higher variance than the dual estimator. Presumably, the Lagrangian does not benefit from the doubly robust property, since both solutions are biased in this practical setting.

\begin{figure}[t]
\centering
  \includegraphics[width=1.\linewidth]{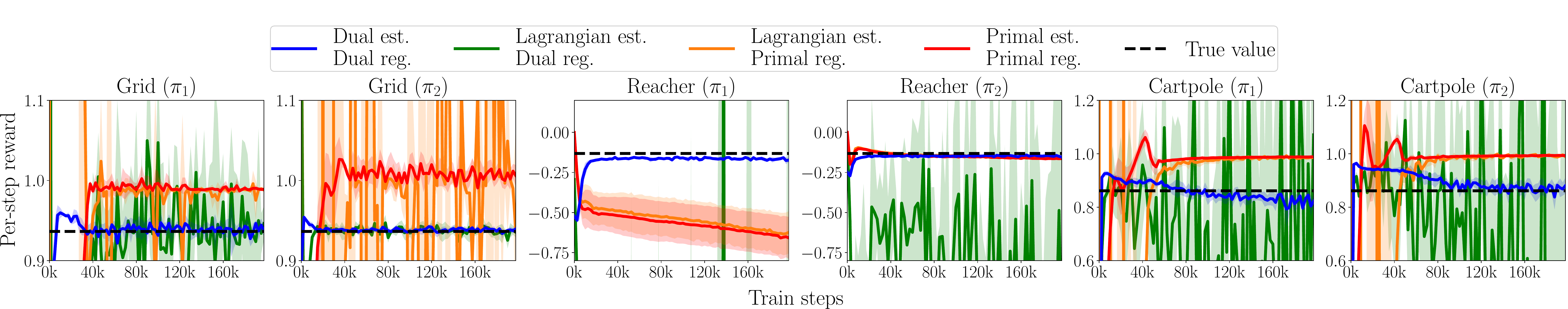}
  \caption{Estimation results on Grid, Reacher, and Cartpole using data collected from different behavior policies ($\pi_2$ is closer to the target policy than $\pi_1$). Biased estimator-regularizer combinations from~\thmref{thm:reg_profile} are omitted. The dual estimator with regularized dual variable outperforms all other estimators/regularizers. Lagrangian can be as good as the dual but has a larger variance.}
  \label{fig:est_reg}
\end{figure}
\begin{figure}[t]
\centering
  \begin{subfigure}{1.\columnwidth}
    \includegraphics[width=1.\linewidth]{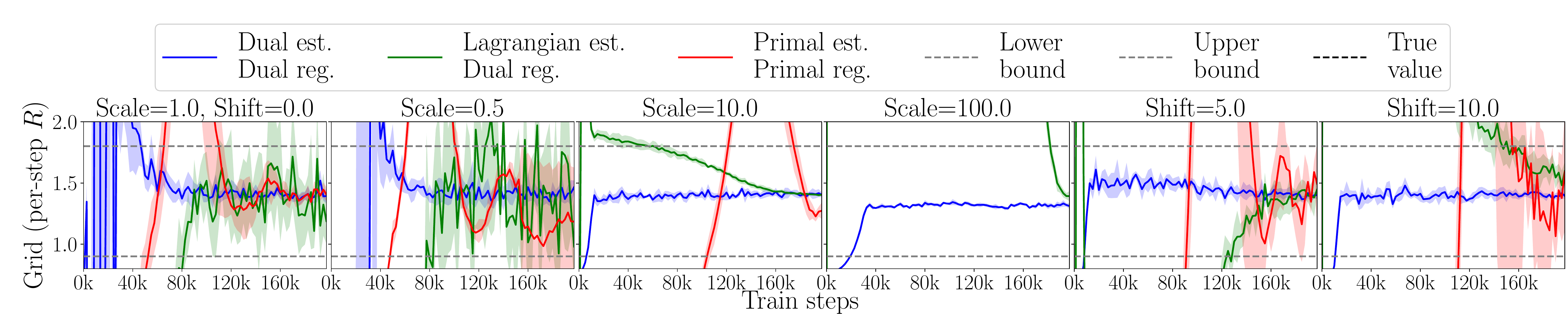}
  \end{subfigure}
  \begin{subfigure}{1.\columnwidth}
    \includegraphics[width=1.\linewidth]{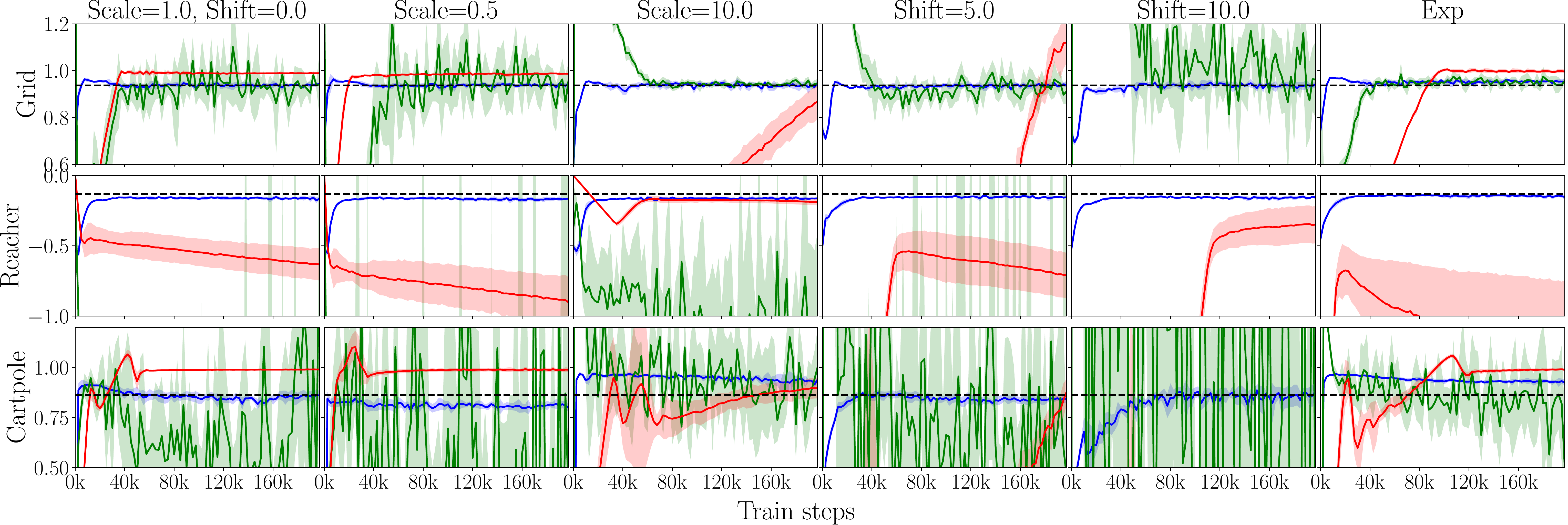}
  \end{subfigure}
  \caption{Primal (red), dual (blue), and Lagrangian (green) estimates under linear (top) and neural network (bottom) parametrization when rewards are transformed during training. Estimations are transformed back and plotted on the original scale. The dual estimates are robust to all transformations, whereas the primal and Lagrangian estimates are sensitive to the reward values.}
\label{fig:est_robust}  
\end{figure}

To more extensively evaluate the dual estimator,
we investigate its performance when the reward function
is scaled by a constant, shifted by a constant, or exponentiated.
\footnote{Note this is separate from $\alpha_R$, which only affects
optimization. We use $\alpha_R=1$ exclusively here.}
To control for difficulties in optimization, we first parametrize the primal and dual variables as linear functions, and use stochastic gradient descent to solve the convex-concave minimax objective in~\eqref{eq:lagrangian_reg} with $\alpha_Q = 0$, $\alpha_\zeta = 1$, and $\alpha_R = 1$. 
Since a linear parametrization changes the ground truth of evaluation,
we compute the upper and lower estimation bounds
by only parameterizing the primal or the dual variable
as a linear function.~\figref{fig:est_robust} (top) shows the estimated per-step reward of the Grid task.
When the original reward is used (col. 1), the primal, dual, and Lagrangian estimates eventually converge to roughly the same value (even though primal estimates converge much slower). When the reward is scaled by 10 or 100 times or shifted by 5 or 10 units (the original reward is between 0 and 1), the resulting primal estimates are severely affected and do not converge given the same number of gradient updates. 
When performing this same evaluation with neural network parametrization (\figref{fig:est_robust}, bottom), the primal estimates continue to exhibit sensitivity to reward transformations, whereas the dual estimates stay roughly the same after being transformed back to the original scale. We further implemented target network for training stability of the primal variable, and the same concolusion holds (see Appendix). Note that while the dual solution is robust to the scale and range of rewards, the optimization objective used here still has $\alpha_R = 1$, which is different from $\alpha_R = 0$ where $\hat\rho_Q$ is no longer a valid estimator.

\subsection{Choice of Regularization ($\alpha_\zeta$, $\alpha_R$, and $\alpha_Q$)}
\label{sec:exp_reg}
Next, we study the choice between regularizing the primal or dual variables.
Given the results of~\secref{sec:exp_est}, we focus on ablations using the dual estimator $\hat\rho_\zeta$  to estimate $\rho_\pi$.
Results are presented in~\figref{fig:reg}. 
As expected, we see that regularizing the primal variables when $\alpha_R = 1$ leads to a biased estimate, especially in Grid ($\pi_1$), Reacher ($\pi_2$), and Cartpole.
Regularizing the dual variable (blue lines) on the other hand does not incur
such a bias.
Additionally, the value of $\alpha_R$ has little effect on the final estimates
when the dual variable is regularized (dotted versus solid blue lines).
While the invariance to $\alpha_R$ may not generalize to other tasks,
an advantage of the dual estimates with regularized dual variable
is the flexibility to set $\alpha_R=0$ or $1$ depending on the 
reward function.

\begin{figure}[t]
\centering
\includegraphics[width=1.\linewidth]{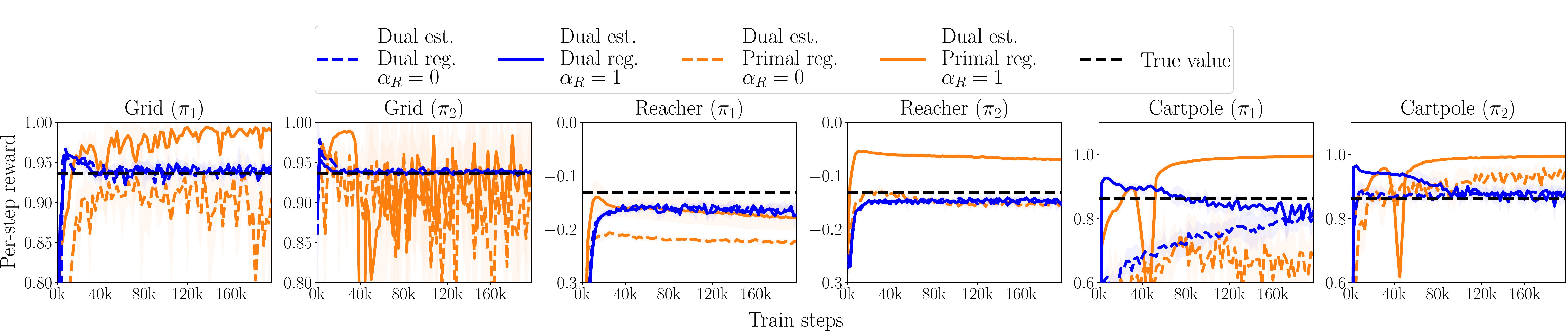}
\caption{Dual estimates when $\alpha_R = 0$ (dotted line) and $\alpha_R = 1$ (solid line). Regularizing the dual variable (blue) is consistently better than regularizing the primal variable (orange). $\alpha_R \ne 0$ and $\alpha_Q \ne 0$ leads to biased estimation (solid orange). The value of $\alpha_R$ does not affect the final estimate when $\alpha_\zeta = 1, \alpha_Q = 0$.}
\label{fig:reg}
\end{figure}
\begin{figure}[h!]
  \begin{subfigure}{1.\columnwidth}
    \includegraphics[width=1.\linewidth]{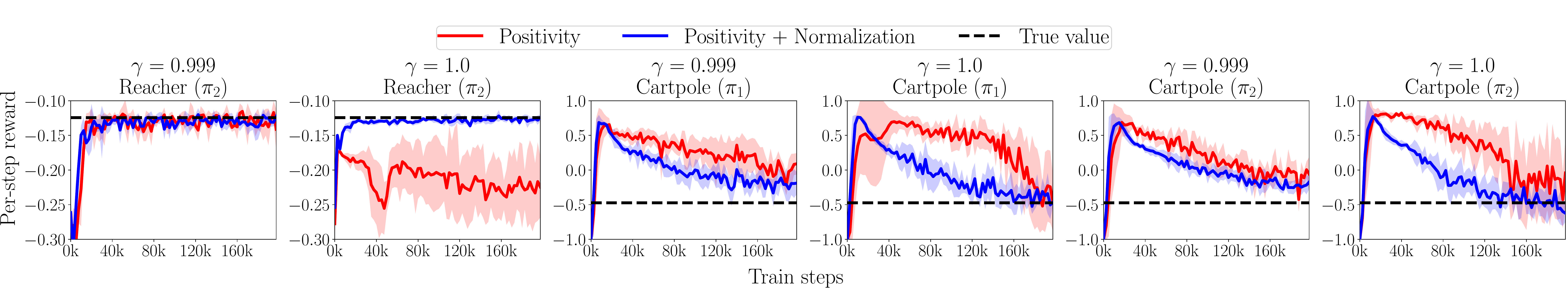}    
  \end{subfigure}
  \begin{subfigure}{1.\columnwidth}
    \includegraphics[width=1.\linewidth]{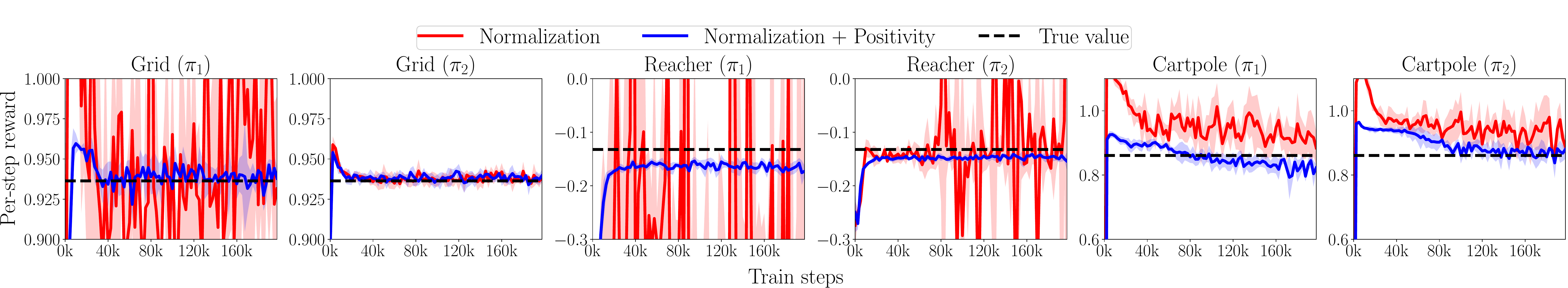}
  \end{subfigure}
  \begin{subfigure}{1.\columnwidth}
    \includegraphics[width=1.\linewidth]{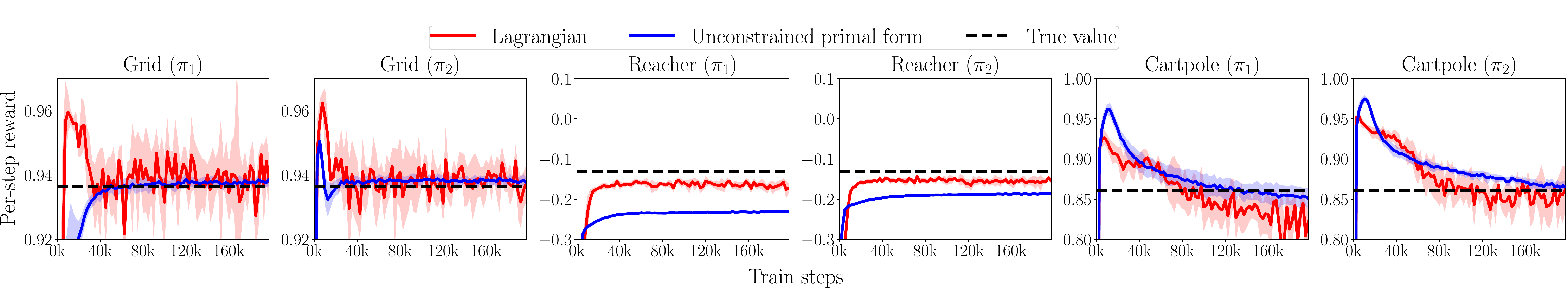}
  \end{subfigure}
  \caption{Apply positive constraint, normalization constraint, and the unconstrained primal form during optimization (blue curves). Positivity constraint (row 1) improves training stability. Normalization constraint is essential when $\gamma = 1$, and also helps when $\gamma < 1$ (row 2). Solving the unconstrained primal problem (row 3) can be useful when the action space is discrete.}
\label{fig:opt}  
\end{figure}

\subsection{Choice of Redundant Constraints ($\lambda$ and $\zeta\ge0$)}
\label{sec:exp_opt}

So far our experiments with the dual estimator used $\lambda$ and $\zeta\ge0$ in the optimizations,
corresponding
to the normalization and positive constraints in the $d$-LP.
However, these are in principle not necessary when $\gamma<1$, and so we evaluate the effect of removing them. 
Given the results of the previous sections, we focus our ablations on the use of the dual estimator $\hat\rho_\zeta$ with dual regularization $\alpha_\zeta>0,\alpha_R=1$.

\textbf{Normalization.}
We consider the effect of removing the normalization constraint ($\lambda$).
~\figref{fig:opt} (row 1) shows the effect of keeping (blue curve) or removing
(red curve) this
constraint during training. 
We see that training becomes less stable and approximation error increases, even when $\gamma<1$.

\textbf{Positivity.} 
We continue to evaluate the effect of removing the positivity constraint $\zeta\ge0$, which, in our previous experiments, was enforced via applying a square function to the dual variable neural network output.
Results are presented in~\figref{fig:opt} (row 2), where we again see that the removal of this constraint is detrimental to optimization stability and estimator accuracy.

\subsection{Choice of Optimization (Lagrangian or Unconstrained Primal Form)}
So far, our experiments have used minimax optimization via the Lagrangian to learn primal and dual variables.
We now consider solving the unconstrained primal form of the $d$-LP, which~\secref{sec:reg_lagrangian} suggests
may lead to an easier, but biased, optimization.~\figref{fig:opt} (row 3)
indeed shows that the unconstrained primal reduces variance on Grid and produces better estimates on Cartpole.
Both environments have discrete action spaces.
Reacher, on the other hand, has a continuous action space,
which creates difficulty when taking the expectation over next step samples, causing bias in
the unconstrained primal form. Given this mixed performance, we generally advocate for
the Lagrangian, unless the task is discrete-action  and the stochasticity of the dynamics is  known to be low.

\section{Related work}\label{sec:related}
Off-policy evaluation has long been studied
in the RL literature~\citep{farajtabar2018more,jiang2015doubly,kallus2019double,munos2016safe,Precup00ET,Thomas15HCPE}.
While some approaches are model-based~\citep{fonteneau13batch}, or work by estimating the value function~\citep{duan2020minimaxoptimal},
most rely on importance reweighting to transform the off-policy
data distribution to the on-policy target distribution.
They often require to know or estimate the behavior policy,
and suffer a variance exponential in the horizon,
both of which limit their applications.
Recently, a series of works were proposed to address these
challenges~\cite{kallus2019efficiently,liu2018breaking,tang20doubly}.  Among them is the DICE family~\citep{NacChoDaiLi19,zhang2020gendice,zhang2020gradientdice},
which performs some form of stationary distribution estimation.
The present paper develops a convex duality framework that unifies
many of these algorithms, and offers further important insights.
Many OPE algorithms may be understood to correspond to the categories considered here.  Naturally, the recent stationary distribution correction algorithms~\citep{NacChoDaiLi19,zhang2020gendice,zhang2020gradientdice}, are the dual methods.
The FQI-style estimator~\citep{duan2020minimaxoptimal} loosely corresponds to our primal estimator. 
Moreover, Lagrangian-type estimators are also considered~\citep{tang20doubly,uehara2019minimax}, although some are not for the behavior-agnostic setting~\citep{tang20doubly}.

Convex duality has been widely used in machine learning, and in RL in particular.
In one line of literature, it was used to solve the Bellman equation, whose fixed point is the value
function~\citep{dai18sbeed,du2017stochastic,LiuLiuGhaMahetal15}. 
Here, duality facilitates derivation of an objective function
that can be conveniently approximated by sample averages, so that solving
for the fixed point is converted to that of finding a saddle point.
Another line of work, more similar to the present paper, is to optimize
the Lagrangian of the linear program that characterizes the value
function~\citep{basserrano2019faster,chen2018scalable,wang2017randomized}.
In contrast to our work, these algorithms typically do not incorporate
off-policy correction, but assume the availability of on-policy samples.

\section{Conclusion}\label{sec:conclusion}
We have proposed a unified view of off-policy evaluation via the regularized Lagrangian of the $d$-LP. Under this unification, existing DICE algorithms are recovered by specific (suboptimal) choices of regularizers, (redundant) constraints, and ways to convert optimized solutions to policy values. By systematically studying the mathematical properties and empirical effects of these choices, we have found that the dual estimates (i.e., policy value in terms of the state-action distribution) offer greater flexibility in incorporating optimization stablizers while preserving asymptotic unibasedness, in comparison to the primal estimates (i.e., estimated $Q$-values). Our study also reveals alternative estimators not previously identified in the literature that exhibit improved performance. Overall, these findings suggest promising new directions of focus for OPE research in the offline setting.

\subsubsection*{Acknowledgments}
We thank Hanjun Dai and other members of the Google Brain team for helpful discussions.

\bibliography{main}
\bibliographystyle{plainnat}

\clearpage
\newpage
\begin{appendix}
\thispagestyle{plain}
\begin{center}
{\huge Appendix}
\end{center}

\section{Robustness Justification}\label{appendix:primal_robust}

We explain the robustness interpretation of the dual regularization as the perturbation of Bellman differences. In this section, we elaborate the robustness interpretation of the primal regularization. For simplicity, we also consider $f_1\rbr{\cdot} = \rbr{\cdot}^2$. Therefore, we have $\preg\cdot\E_{\sampone}[f_1(\qvar(s,a))] = \preg\cdot\cbr{\max_{\delta\rbr{s, a}} \inner{Q}{\delta}- \E_{\sampone}\sbr{\delta^2\rbr{s, a}} }$. Plug the dual form into~\eqref{eq:lagrangian_reg} and with strong duality, we have
\begin{align}
\label{eq:ro_lagrangian_preg}
\nonumber \max_{\zeta\positive, \delta} \min_{\qvar, \norm} L_D(\zeta, \qvar, \norm, \delta) \defeq  & (1-\gamma)\cdot\E_{\initsamp}[\qvar(s_0,a_0)] + \preg \E_{\sampone}\sbr{\delta\rbr{s, a}\cdot Q\rbr{s, a}} + \norm \\
\nonumber &+ \E_{\samptwo}[\zeta(s,a)\cdot(\reward\cdot R(s,a) + \gamma \qvar(s',a') - \qvar(s,a) - \norm)] \\
  &- \preg\cdot\E_{\sampone}[\delta^2(s,a)] - \dreg\cdot \E_{\sampone}[f_2(\zeta(s,a))],
\end{align}
which can be understood as the Lagrangian of
\begin{eqnarray}
\max_{\zeta\positive, \delta}\,\, &\reward\E_{\sampone}\sbr{\zeta\rbr{s, a}\cdot R\rbr{s, a}} - \preg\cdot\E_{\sampone}[\delta^2(s,a)] - \dreg\cdot \E_{\sampone}[f_2(\zeta(s,a))] \nonumber \\
\st& \rbr{1 - \gamma} \mu_0\pi + \preg \visitrb\cdot\delta + \bellmant\cdot\rbr{\visitrb\cdot \zeta} = \rbr{\visitrb\cdot \zeta} \label{eq:relax_dual}\\
& \E_{\sampone}\sbr{\zeta} = 1. \nonumber
\end{eqnarray}
As we can see, the primal regularization actually introduces $L_2$-ball perturbations to the stationary state-action distribution condition~\eqref{eq:relax_dual}. For different regularization, the perturbations will be in different dual spaces. For examples, with entropy-regularization, the perturbation lies in the simplex. The corresponding optimization of~\eqref{eq:ro_lagrangian_dreg} is
\begin{eqnarray}\label{eq:relax_primal}
\min_{Q}\,\, &\rbr{1 - \gamma}\E_{\mu_0\pi}\sbr{Q\rbr{s, a}} + \preg\cdot \E_{\sampone}\sbr{f_1\rbr{Q}} + \dreg\cdot \E_{\sampone}\sbr{\delta^2\rbr{s, a}}\\
\st& Q\rbr{s, a}\ge R\rbr{s, a} +\bellman Q\rbr{s, a} - \dreg\delta\rbr{s, a}.
\end{eqnarray}
In both~\eqref{eq:relax_primal} and~\eqref{eq:relax_dual}, the relaxation of dual $\zeta$ in~\eqref{eq:relax_dual} does not affect the optimality of dual solution: the stationary state-action distribution is still the only solution to~\eqref{eq:relax_dual}; while in~\eqref{eq:relax_primal}, the relaxation of primal $Q$ will lead to different optimal primal solution. From this view, one can also justify the advantages of the dual OPE estimation.

\section{Proof for~\thmref{thm:reg_profile}}\label{appendix:reg_prof}
The full enumeration of $\preg,\dreg,\reward,\norm,$ and $\zeta\positive$ results in $2^5=32$ configurations. 
  We note that it is enough to characterize the solutions $Q^*,\zeta^*$ under these different configurations. Clearly, the primal estimator $\hat{\rho}_Q$ is unbiased when $Q^*=\qpi$, and the dual estimator $\hat{\rho}_\zeta$ is unbiased when $\zeta^*=\visitpi/\visitrb$.
  For the Lagrangian estimator $\hat{\rho}_{Q,\zeta}$, we may write it in two ways:
  \begin{align}
    \label{eq:lagrange-q}
    \hat{\rho}_{Q,\zeta}(\pi) &= \hat{\rho}_Q(\pi) + \sum_{s,a} \visitrb(s,a)\zeta(s,a)(R(s,a) + \gamma\bellmannog Q(s,a) - Q(s,a)) \\
    &= \hat{\rho}_{\zeta}(\pi) + \sum_{s,a} Q(s,a)((1-\gamma)\init(s)\pi(a|s) + \gamma\bellmantnog \visitrb\times\zeta(s,a) - \visitrb\times\zeta(s,a)).
    \label{eq:lagrange-z}
  \end{align}
  It is clear that when $Q^*=\qpi$, the second term of~\eqref{eq:lagrange-q} is 0 and $\hat{\rho}_{Q,\zeta}(\pi) = \rho(\pi)$. When $\zeta^*=\visitpi/\visitrb$, the second term of~\eqref{eq:lagrange-z} is 0 and $\hat{\rho}_{Q,\zeta}(\pi) = \rho(\pi)$. Therefore, the Lagrangian estimator is unbiased when 
  either $Q^*=\qpi$ or $\zeta^*=\visitpi/\visitrb$.

  Now we continue to characterizing $Q^*,\zeta^*$ under different configurations.
  First, when $\preg=0, \dreg=0$, it is clear that the solutions are always unbiased by virtue of~\thmref{thm:dual-succinct} (see also~\cite{algae}). When $\preg>0,\dreg>0$, the solutions are in general biased. We summarize the remaining configurations (in the discounted case) of $\preg>0,\dreg=0$ and $\preg=0,\dreg>0$ in the table below. We provide proofs for the configurations of the shaded cells. Proofs for the rest configurations can be found in~\cite{NacChoDaiLi19,algae}.
\begin{table}[h]
\caption{Optimal solutions for all configurations. Configurations with new proofs are shaded in gray.}
\renewcommand{\arraystretch}{1.2}
\setlength{\tabcolsep}{1pt}
\centering
\tiny
\begin{tabular}{c|c|c|c||c|c|c}
    \toprule        
\multicolumn{3}{l|}{ Regularizer (w./w.o. $\norm$) } &Case& $Q^*(s,a)$ & $\zeta^*(s,a)$ & $L(Q^*,\zeta^*)$\\\hline\hline

\rowcolor{gray!15}& \multirow{2}{*}{ $\reward=1$ } & $\zeta$ free &i & $Q^\pi$ & $\frac{d^\pi}{\visitrb}+\preg\frac{\rbr{\Ical - \gamma\Pcal_*^\pi}^{-1}\rbr{\visitrb\cdot f_1^{\prime}\rbr{Q^\pi}}}{\visitrb}$ & \shortstack{$\reward(1-\gamma)\cdot\E_{\mu_0}[\qpi] $\\$+ \preg\E_{(s,a)\sim \visitrb}\sbr{f_1\rbr{{Q^\pi}}}$} \\\hhline{>{\arrayrulecolor{gray!15}}-->{\arrayrulecolor{black}}-----}

\rowcolor{gray!15} $\dreg$ = 0 & & $\zeta\positive$ &ii& \shortstack{ $f_1^{*\prime}\bigg(\frac{1}{\preg}\bigg(\big(\preg f_1'\rbr{Q^\pi} + $\\$ \frac{\rbr{1 - \gamma}\mu_0\pi}{\visitrb} \big)_+ - \frac{\rbr{1 - \gamma}\mu_0\pi}{\visitrb}\bigg)\bigg)$ } & \shortstack{$\frac{1}{\visitrb}\rbr{\Ical - \bellman}^{-1}\cdot$\\$\visitrb\rbr{\preg f_1'\rbr{Q^\pi} + \frac{\rbr{1 - \gamma}\mu_0\pi}{\visitrb} }_+$} & \shortstack{$(1-\gamma)\cdot\E_{\mu_0}[\qvar^*] $\\$+ \E_{\visitrb}[\zeta^*(s,a)\cdot(\reward\cdot r $\\$ + \gamma \qvar^*(s',a') - \qvar^*(s,a))] $\\$+ \preg\cdot\E_{\visitrb}[f_1(\qvar^*(s,a))]$}\\\hhline{>{\arrayrulecolor{gray!15}}->{\arrayrulecolor{black}}------}

\rowcolor{gray!15} $\preg$ > 0 & & $\zeta$ free &iii&  & \multirow{6}{*}{ $\frac{d^\pi}{\visitrb}$~\citep{NacChoDaiLi19,algae}} & \\\hhline{>{\arrayrulecolor{gray!15}}-->{\arrayrulecolor{black}}-->{\arrayrulecolor{gray!15}}---}

\rowcolor{gray!15} & \multirow{-2}{*}{ $\reward=0$ } & $\zeta\positive$ &iv& \multirow{-2}{*}{$f^{*\prime}_1\rbr{0 }$} & & \multirow{-2}{*}{$-\preg f_1^*\rbr{0}$}\\\hhline{>{\arrayrulecolor{black}}----->{\arrayrulecolor{gray!15}}->{\arrayrulecolor{black}}-}

& \multirow{2}{*}{ $\reward=1$ } & $\zeta$ free &v&  & \\\cline{3-4}
$\dreg$ > 0 & & $\zeta\positive$ &vi& $-\dreg \rbr{\Ical - \Pcal^\pi}^{-1} f_2'(\frac{d^\pi}{\visitrb})$ && $\reward \cdot \E_{(s,a,r,s')\sim\visitrb}[r]$\\\cline{2-4}
$\preg$ = 0 & \multirow{2}{*}{ $\reward=0$ } & $\zeta$ free &vii&{$+\reward Q^\pi$ ~\citep{NacChoDaiLi19,algae}} && { $-\dreg\cdot D_f(d^\pi\|\visitrb)$~\citep{NacChoDaiLi19,algae}}\\\cline{3-4}
& & $\zeta\positive$ &viii& & \\
\bottomrule
\end{tabular}
\end{table}

\begin{proof}
Under our Assumptions~\ref{asmp:mdp_reg} and~\ref{asmp:bounded_ratio}, the strong duality holds for~\eqref{eq:lagrangian_reg}. We provide the proofs by checking the configurations case-by-case.

\begin{itemize}
  \item {\bf iii)-iv)} In this configuration, the regularized Lagrangian~\eqref{eq:lagrangian_reg} becomes
  \begin{eqnarray}\label{eq:lagrangian_primal_reg_noreward}
  \nonumber \max_{\zeta\positive} \min_{\qvar, \norm} L_D(\zeta, \qvar, \norm) \defeq  & (1-\gamma)\cdot\E_{\initsamp}[\qvar(s_0,a_0)] + \preg\cdot\E_{\sampone}[f_1(\qvar(s,a))] + \norm \\
  \nonumber &+ \E_{\samptwo}[\zeta(s,a)\cdot(\gamma \qvar(s',a') - \qvar(s,a) - \norm)],
  \end{eqnarray}
  which is equivalent to 
  \begin{eqnarray}\label{eq:lagrangian_primal_reg_noreward2}
  \nonumber \max_{\zeta\positive} \min_{\qvar} L_D(\zeta, \qvar) = 
  & \inner{\rbr{1 - \gamma}\mu_0\pi +  \bellmant\cdot\rbr{\visitrb\cdot \zeta} - \visitrb\cdot\zeta}{Q} + \preg\EE_{\visitrb}\sbr{f_1\rbr{Q}}\\
  &\st\quad \EE_{\visitrb}\sbr{\zeta} = 1.
  \end{eqnarray}

  Apply the Fenchel duality w.r.t. $Q$, we have 
  \begin{eqnarray}
  \max_\zeta& L_D\rbr{\zeta, Q^*} = -\preg\EE_{\visitrb}\sbr{ f_1^*  \rbr{\frac{ \rbr{1 - \gamma}\mu_0\pi +  \bellmant\cdot\rbr{\visitrb\cdot \zeta} - \visitrb\cdot\zeta }{\preg\visitrb} }}\\
  \st&\quad \EE_{\visitrb}\sbr{\zeta} = 1.
   \end{eqnarray}
  If $f^*_1\rbr{\cdot}$ achieves the minimum at zero, it is obvious that 
  $$
  \visitrb\cdot \zeta^* = \rbr{1 - \gamma}\mu_0\pi +  \bellmant\cdot\rbr{\visitrb\cdot \zeta^*}\Rightarrow \visitrb\cdot\zeta^* = d^\pi.
  $$
  Therefore, we have
  $$
  L\rbr{\zeta^*, Q^*} = -\preg f_1^*\rbr{0},
  $$
  and 
  \begin{eqnarray*}
  Q^* &=& \argmax_{Q} \inner{\rbr{1 - \gamma}\mu_0\pi +  \bellmant\cdot\rbr{\visitrb\cdot \zeta^*} - \visitrb\cdot\zeta^*}{Q} + \preg\EE_{\visitrb}\sbr{f_1\rbr{Q}}\\
  &=& f^{*\prime}_1\rbr{0 }
  \end{eqnarray*}

  \item {\bf i)-ii)} Following the derivation in case {\bf iii)-iv)}, we have the regularized Lagrangian as almost the same as~\eqref{eq:lagrangian_primal_reg_noreward} but has an extra term $\reward\EE_{\visitrb}\sbr{\zeta\cdot R}$, \ie
  \begin{eqnarray}\label{eq:lagrangian_primal_reg}
  \nonumber \max_{\zeta} \min_{\qvar} L_D(\zeta, \qvar) \defeq  & (1-\gamma)\cdot\E_{\initsamp}[\qvar(s_0,a_0)] + \preg\cdot\E_{\sampone}[f_1(\qvar(s,a))] \\
  \nonumber &+ \E_{\samptwo}[\zeta(s,a)\cdot(\reward\cdot R\rbr{s, a} + \gamma \qvar(s',a') - \qvar(s,a))].
  \end{eqnarray}
  We first consider the case where the $\zeta$ is free and the normalization constraint is not enforced.

  After applying the Fenchel duality w.r.t. $Q$, we have  
  \begin{eqnarray}
  \max_\zeta& L_D\rbr{\zeta, Q^*} = \reward\inner{\visitrb\cdot\zeta}{R} -\preg\EE_{\visitrb}\sbr{ f_1^*  \rbr{\frac{  \visitrb\cdot\zeta  -\rbr{1 - \gamma}\mu_0\pi -  \bellmant\cdot\rbr{\visitrb\cdot \zeta} }{\preg\visitrb} }}.
  \end{eqnarray}  
  We denote 
  \begin{align*}
  & \nu = \frac{ \visitrb\cdot\zeta  - \rbr{1 - \gamma}\mu_0\pi -  \bellmant\cdot\rbr{\visitrb\cdot \zeta} }{\visitrb} \\
  \Rightarrow \,\,\,\, & \visitrb\cdot \zeta = \rbr{\Ical - \bellmant}^{-1}\rbr{\rbr{1 - \gamma}\mu_0\pi + \visitrb\cdot\nu},
  \end{align*}
  and thus, 
  \begin{eqnarray*}
  \lefteqn{
  L_D\rbr{\zeta^*, Q^*} = \max_{\nu}\inner{\rbr{\Ical - \bellmant}^{-1}\rbr{\rbr{1 - \gamma}\mu_0\pi + \visitrb\cdot\nu}}{\reward R} - \preg \EE_{\visitrb}\sbr{f_1^*\rbr{\frac{\nu}{\preg}}}} \\
  &=& \reward \rbr{1 - \gamma}\EE_{\initsamp}\sbr{\qvar^\pi\rbr{s_0, a_0}} + \max_\nu \EE_{\visitrb}\sbr{\nu\cdot\rbr{\qvar^\pi}} - \preg \EE_{\visitrb}\sbr{f_1^*\rbr{\frac{\nu}{\preg}}},\\
  &=& \reward\rbr{1 - \gamma}\EE_{\initsamp}\sbr{\qvar^\pi\rbr{s_0, a_0}} + \preg\EE_{\visitrb}\sbr{f_1\rbr{{\qvar^\pi}}}
  \end{eqnarray*}
  where the second equation comes from the fact $\qvar^\pi = \rbr{\Ical-\bellman}^{-1}R$ and last equation comes from Fenchel duality with $\nu^* = \preg {f_1^{\prime}}^{}\rbr{\qvar^\pi}$. 

  Then, we can characterize 
  \begin{eqnarray*}
  \zeta^* &=& \frac{\rbr{\Ical - \bellmant}^{-1}\rbr{\rbr{1 - \gamma}\mu_0\pi }}{\visitrb} + \preg\frac{\rbr{\Ical - \bellmant}^{-1}\rbr{\visitrb\cdot f'_1\rbr{\qvar^\pi} }}{\visitrb}\\
  &=& \frac{d^\pi}{\visitrb} + \preg\frac{\rbr{\Ical - \bellmant}^{-1}\rbr{\visitrb\cdot f'_1\rbr{\qvar^\pi} }}{\visitrb},
  \end{eqnarray*}
  and
  \begin{eqnarray*}
  Q^* = \rbr{f_1'}^{-1}\rbr{\frac{  \visitrb\cdot\zeta^*  -\rbr{1 - \gamma}\mu_0\pi -  \bellmant\cdot\rbr{\visitrb\cdot \zeta^*}  }{\preg\visitrb}} = Q^{\pi}.
  \end{eqnarray*}

  If we have the positive constraint, \ie, $\zeta\ge 0$, we denote
  $$
  \exp\rbr{\nu} = \frac{\rbr{\Ical - \bellmant}\rbr{\visitrb\cdot\zeta}}{\visitrb}\Rightarrow \visitrb\cdot\zeta = \rbr{\Ical - \bellmant}^{-1}\visitrb\cdot \exp\rbr{\nu},
  $$
  then, 
  \begin{eqnarray*}
  L_D\rbr{\zeta^*, Q^*} = \max_{\nu} \EE_{\visitrb}\sbr{\exp\rbr{\nu}\cdot Q^\pi} - \preg \EE_{\visitrb}\sbr{f_1^*\rbr{\frac{1}{\preg}\rbr{\exp\rbr{\nu} - \frac{\rbr{1 - \gamma}\mu_0\pi}{\visitrb}}}}.
  \end{eqnarray*}
  By first-order optimality condition, we have
  \begin{eqnarray}\label{eq:zeta_positive_preg_reward}
  &&\exp\rbr{\nu^*}\rbr{Q^\pi - f_1^{*\prime}\rbr{\frac{1}{\preg}\rbr{\exp\rbr{\nu} - \frac{\rbr{1 - \gamma}\mu_0\pi}{\visitrb}}}} = 0
  \nonumber\\
  &=& \exp\rbr{\nu^*} = \rbr{\preg f_1'\rbr{Q^\pi} + \frac{\rbr{1 - \gamma}\mu_0\pi}{\visitrb} }_+ \nonumber \\
  &\Rightarrow& \visitrb\cdot \zeta^* = \rbr{\Ical - \bellman}^{-1}\cdot\visitrb\rbr{\preg f_1'\rbr{Q^\pi} + \frac{\rbr{1 - \gamma}\mu_0\pi}{\visitrb} }_+ \nonumber \\
  &\Rightarrow& \zeta^* = \frac{1}{\visitrb}\rbr{\Ical - \bellman}^{-1}\cdot\visitrb\rbr{\preg f_1'\rbr{Q^\pi} + \frac{\rbr{1 - \gamma}\mu_0\pi}{\visitrb} }_+.
  \end{eqnarray}
  For $Q^*$, we obtain from the Fenchel duality relationship,
  \begin{eqnarray}\label{eq:q_positive_preg_reward}
  Q^* &=& f_1^{*\prime}\rbr{\frac{1}{\preg}\rbr{\exp\rbr{\nu^*} - \frac{\rbr{1 - \gamma}\mu_0\pi}{\visitrb}}}\nonumber\\
  &=& f_1^{*\prime}\rbr{\frac{1}{\preg}\rbr{\rbr{\preg f_1'\rbr{Q^\pi} + \frac{\rbr{1 - \gamma}\mu_0\pi}{\visitrb} }_+ - \frac{\rbr{1 - \gamma}\mu_0\pi}{\visitrb}}}.
  \end{eqnarray}
  Then, the $L_D\rbr{\zeta^*, Q^*}$ can be obtained by plugging $\rbr{\zeta^*, Q^*}$ in~\eqref{eq:zeta_positive_preg_reward} and~\eqref{eq:q_positive_preg_reward}. Obviously, in this case, the estimators are all biased. 

  As we can see, in both {\bf i)} and {\bf ii)}, none of the optimal dual solution $\zeta^*$ satisfies the normalization condition. Therefore, with the extra normalization constraint, the optimization will be obviously biased.

  \item {\bf v)-viii)} These cases are also proved in~\citet{algae} and we provide a more succinct proof here. In these configurations, whether $\reward$ is involved or not does not affect the proof. We will keep this component for generality. We ignore the $\zeta\positive$ and $\norm$ for simplicity, the conclusion does not affected, since the optimal solution $\zeta^*$ automatically satisfies these constraints.  

  Consider the regularized Lagrangian~\eqref{eq:lagrangian_reg} with such configuration, we have
  \begin{eqnarray}\label{eq:lagrangian_dual_reg}
  \nonumber \min_{\qvar}\max_{\zeta}\,\,  L_D(\zeta, \qvar) \defeq  & (1-\gamma)\cdot\E_{\initsamp}[\qvar(s_0,a_0)] -\dreg\cdot \E_{\sampone}[f_2(\zeta(s,a))] \\
  &+ \E_{\samptwo}[\zeta(s,a)\cdot(\reward\cdot R(s,a) + \gamma \qvar(s',a') - \qvar(s,a) )].
  \end{eqnarray}
  Apply the Fenchal duality to $\zeta$, we obtain
  \begin{equation}\label{eq:lagrangian_dual_fenchel}
  \min_{\qvar}\,\, L_D\rbr{\zeta^*, \qvar}\defeq (1-\gamma)\cdot\E_{\initsamp}[\qvar(s_0,a_0)] + \dreg\EE_{\visitrb}\sbr{f_2^*\rbr{\frac{1}{\dreg} \rbr{\Bcal^\pi\cdot Q\rbr{s, a} - Q\rbr{s, a}}}},
  \end{equation}
  with $\Bcal^\pi\cdot Q\rbr{s, a}\defeq \reward\cdot R\rbr{s, a} + \gamma \Pcal^\pi Q\rbr{s, a}$. We denote $\nu\rbr{s, a} = \Bcal\cdot Q\rbr{s, a} - Q\rbr{s, a}$, then, we have
  $$
  Q\rbr{s, a} = \rbr{\Ical - \bellman}^{-1}\rbr{\reward\cdot R - \nu}. 
  $$
  Plug this into~\eqref{eq:lagrangian_dual_fenchel}, we have
  \begin{eqnarray}
\nonumber L_D\rbr{\zeta^*, \qvar^*}&=&\min_\nu\,\,  (1-\gamma)\cdot\E_{\initsamp}\sbr{\rbr{\rbr{\Ical - \bellman}^{-1}\rbr{\reward\cdot R-\nu}}(s_0,a_0)} \\\nonumber
  &&+ \dreg\EE_{\visitrb}\sbr{f_2^*\rbr{\frac{1}{\dreg} \nu\rbr{s, a}}},\\\nonumber
  &=& \reward\EE_{d^\pi}\sbr{R\rbr{s, a}}- \dreg\max_{\nu}\rbr{\E_{d^\pi}\sbr{\frac{\nu(s_0,a_0)}{\dreg} } + \EE_{\visitrb}\sbr{f_2^*\rbr{\frac{1}{\dreg} \nu\rbr{s, a}}}},\\
  &=& \reward\EE_{d^\pi}\sbr{R\rbr{s, a}} - \dreg D_f\rbr{d^\pi||\visitrb}
  \end{eqnarray}
  The second equation comes from the fact $d^\pi = \rbr{\Ical - \bellmant}^{-1}\rbr{\init\pi}$. The last equation is by the definition of the Fenchel duality of $f$-divergence. Meanwhile, the optimal $\frac{1}{\dreg}\nu^* = f_2^{\prime}\rbr{\frac{d^\pi}{\visitrb}}$. Then, we have
  \begin{eqnarray*}
  \qvar^* &=& -\rbr{\Ical - \bellman}^{-1}\nu^* + \rbr{\Ical - \bellman}^{-1}\rbr{\reward\cdot R}\\
  &=& -\dreg\rbr{\Ical - \bellman}^{-1}f_2^{\prime}\rbr{\frac{d^\pi}{\visitrb}} + \reward Q^\pi,
  \end{eqnarray*}
  and 
  \begin{eqnarray*}
  \zeta^*(s, a) &=& \argmax_{\zeta} \zeta\cdot \nu^*(s, a) - \dreg f_2\rbr{\zeta\rbr{s, a}}\\
  &=& f_2^{*\prime}\rbr{\frac{1}{\dreg}\nu^*\rbr{s, a}} = \frac{d^\pi\rbr{s, a}}{\visitrb\rbr{s, a}}.  
  \end{eqnarray*}

\end{itemize}

\end{proof}

\section{Recovering Existing OPE estimators}\label{appendix:recover_ope}
We verify the LSTDQ as a special case of the unified framework if the primal and dual are linearly parametrized, \ie, $\qvar\rbr{s, a} = w^\top \phi\rbr{s, a}$ and $\tau\rbr{s, a} = v^\top \phi\rbr{s, a}$, from any unbiased estimator  without $\xi\positive$ and $\norm$. For simplicity, we assume the solution exists. 

\begin{itemize}

  \item  When {$\rbr{\preg=1, \dreg=0, \reward=1}$}, we have the estimator as
  \begin{align}
  \nonumber \max_{v} \min_{w} L_D(v, w) \defeq  & (1-\gamma)\cdot w^\top\E_{\initsamp}[\phi(s_0,a_0)]  + \preg\cdot \E_{\sampone}[f_1(w^\top \phi(s,a))] \\
  \nonumber &+ v^\top\E_{\samptwo}[\phi(s,a)\cdot(\reward\cdot R(s,a) + \gamma w^\top \phi(s',a') - w^\top \phi(s,a))].
  \end{align}
  Then, we have the first-order optimality condition for $v$ as
  \begin{eqnarray*}
  && \E_{\samptwo}[\phi(s,a)\cdot(\reward\cdot R(s,a) + \gamma w^\top \phi(s',a') - w^\top \phi(s,a))] = 0,\\
  &\Rightarrow& w = {\underbrace{\E_{\samptwo}[\phi(s,a)\cdot(\phi(s,a) - \gamma  \phi(s',a'))]}_{\Xi}}^{-1}\E_{\rbr{s, a}\sim \visitrb}\sbr{\reward\cdot R\rbr{s, a}\phi\rbr{s, a}},\\
  &\Rightarrow& Q^*\rbr{s, a} = w^\top \phi\rbr{s, a},
  \end{eqnarray*}
  which leads to 
  \begin{eqnarray*}
  \hat{\rho}_Q(\pi) &=& (1-\gamma)\cdot \E_{\initsamp}[\hat\qvar(s_0,a_0)] \\
  &=& \rbr{1 - \gamma} \E_{\initsamp}[\phi\rbr{s, a}]^\top\Xi^{-1} \E_{\rbr{s, a}\sim \visitrb}\sbr{ R\rbr{s, a}\phi\rbr{s, a}}.
  \end{eqnarray*}

  \item  When {$\rbr{\preg=0, \dreg=1, \reward=\{0/1\}}$}, we have the estimator as
  \begin{align}
  \nonumber \max_{v} \min_{w} L_D(v, w) \defeq  & (1-\gamma)\cdot w^\top\E_{\initsamp}[\phi(s_0,a_0)]  - \dreg\cdot \E_{\sampone}[f_2(v^\top \phi(s,a))] \\
  \nonumber &+ v^\top\E_{\samptwo}[\phi(s,a)\cdot(\reward\cdot R(s,a) + \gamma w^\top \phi(s',a') - w^\top \phi(s,a))].
  \end{align}
  Then, we have the first-order optimality condition as
  \begin{eqnarray*}
  v^\top \E_{\samptwo}[\phi(s,a)\cdot(\gamma  \phi(s',a') - \phi(s,a))]  + (1-\gamma)\cdot \E_{\initsamp}[\phi(s_0,a_0)] = 0,
  \end{eqnarray*}
  which leads to 
  \begin{eqnarray*}
  v = (1-\gamma)\cdot \Xi^{-1} \E_{\initsamp}[\phi(s_0,a_0)].
  \end{eqnarray*}
  Therefore, the dual estimator is 
  \begin{eqnarray*}
  \hat\rho_\zeta\rbr{\pi} &=& \E_{(s,a,r)\sim\visitrb}\sbr{R\cdot\phi\rbr{s, a}}^\top v\\
  &=& \rbr{1 - \gamma} \E_{\initsamp}[\phi\rbr{s, a}]^\top\Xi^{-1} \E_{\rbr{s, a}\sim \visitrb}\sbr{ R\rbr{s, a}\phi\rbr{s, a}}.
  \end{eqnarray*}

  \item When $\rbr{\preg=1,\dreg=0, \reward=0}$, by the conclusion for~\eqref{eq:lagrangian_primal_reg_noreward}, we have
  $$
  v^\top \E_{\samptwo}[\phi(s,a)\cdot(\gamma  \phi(s',a') - \phi(s,a))]  + (1-\gamma)\cdot \E_{\initsamp}[\phi(s_0,a_0)] = 0,
  $$
  which leads to similar result as above case. 
\end{itemize}

\section{Alternative Biased Form}\label{appendix:alter_primal}

\paragraph{Unconstrained Primal Forms}

When $\alpha_\zeta>0$ and $\alpha_Q = 0$, the form of the Lagranian can be simplified to yield an optimization over only $Q$. 
Then, we may simplify,
\begin{multline}
  \max_{\zeta(s,a)} \zeta(s,a)\cdot(\alpha_R\cdot R(s,a) + \gamma \bellmannog\qvar(s,a) - \qvar(s,a)) - \alpha_\zeta\cdot f_2(\zeta(s,a)) \\ = \alpha_\zeta\cdot f_2^*\left(\frac{1}{\alpha_\zeta}(\alpha_R\cdot R(s,a) + \gamma\bellmannog\qvar(s,a) - \qvar(s,a))\right).
\end{multline}
So, the Lagrangian may be equivalently expressed as an optimization over only $\qvar$:
\begin{multline}
  \min_{\qvar} (1-\gamma)\cdot\E_{\initsamp}[\qvar(s_0,a_0)] + \alpha_Q\cdot\E_{\sampone}[f_1(\qvar(s,a))] \\
  + \alpha_\zeta\cdot \E_{\sampone}\left[f_2^*\left(\frac{1}{\alpha_\zeta}(\alpha_R\cdot R(s,a) + \gamma\bellmannog\qvar(s,a) - \qvar(s,a))\right)\right].
\end{multline}
We call this the \emph{unconstrained primal form}, since optimization is now exclusively over primal variables.
Still, given a solution $Q^*$, the optimal $\zeta^*$ to the original Lagrangian may be derived as,
\begin{equation}
  \zeta^*(s,a) = f_2^{*\prime}((\alpha_R\cdot R(s,a) + \gamma\bellmannog\qvar^*(s,a) - \qvar^*(s,a)) / \alpha_\zeta).
\end{equation}
Although the unconstrained primal form is simpler, in practice it presents a disadvantage, due to inaccessibility of the transition operator $\bellmannog$. That is, in practice, one must resort to optimizing the primal form as
\begin{multline}
  \min_{\qvar} (1-\gamma)\cdot\E_{\initsamp}[\qvar(s_0,a_0)] + \alpha_Q\cdot\E_{\sampone}[f_1(\qvar(s,a))] \\
  + \alpha_\zeta\cdot \E_{\samptwo}\left[f_2^*\left(\frac{1}{\alpha_\zeta}(\alpha_R\cdot R(s,a) + \gamma\qvar(s',a') - \qvar(s,a))\right)\right].
\end{multline}
This is in general a \emph{biased} estimate of the true objective and thus leads to biased solutions, as the expectation over the next step samples are taken inside a square function (we choose $f_2$ to be the square function). Still, in some cases (e.g., in simple and discrete environments), the bias may be desirable as a trade-off in return for a simpler optimization.

\paragraph{Unconstrained Dual Form}
We have presented an unconstrained primal form. Similarly, we can derive the unconstrianed dual form by removing the primal variable with a particular primal regularization $\alpha_Q \E_{\visitrb}\sbr{f_1\rbr{Q}}$. Then, we can simplify
\begin{align}
\nonumber\min_{Q\rbr{s', a'}} & \frac{1}{\visitrb\rbr{s', a'}}\rbr{1 - \gamma}\init(s')\pi\rbr{a'|s'}\cdot Q\rbr{s', a'}+ \alpha_Q f_1\rbr{Q}\\ 
\nonumber&+\frac{1}{\visitrb\rbr{s', a'}} \rbr{\gamma \int P^\pi\rbr{s', a'|s, a}\visitrb\cdot \zeta\rbr{s, a}dsda - \visitrb\rbr{s', a'}\zeta\rbr{s', a'}}\cdot Q\rbr{s', a'} \\
&= -\alpha_Q\cdot f^*_1\rbr{\frac{\visitrb\cdot\zeta - \rbr{1 - \gamma}\init\pi - \gamma \rbr{\bellmantnog\cdot \visitrb}\zeta}{\alpha_Q\visitrb}},
\end{align}
with ${Q^*} = f^{*\prime}_1\rbr{\frac{\visitrb\cdot\zeta - \rbr{1 - \gamma}\init\pi - \gamma \rbr{\bellmantnog\cdot \visitrb}\zeta}{\alpha_Q\visitrb}}$.

So, the regularized Lagrangian can be represented as
\begin{align}
\nonumber&\max_{d} \reward\EE_{\visitrb}\sbr{\zeta\cdot R}\\
&-\alpha_Q\EE_{\visitrb}\sbr{f^*_1\rbr{\frac{\visitrb\cdot\zeta - \rbr{1 - \gamma}\init\pi - \gamma \rbr{\bellmantnog\cdot \visitrb}\zeta}{\alpha_Q\visitrb}}}- \alpha_\zeta \EE_{\visitrb}\sbr{f_2\rbr{\zeta}}.
\end{align}
Similarly, to approximate the intractable second term, we must use 
\begin{align}
\nonumber&\max_{d} \alpha_R\EE_{\visitrb}\sbr{\zeta\cdot R} \\\nonumber&-\alpha_Q\EE_{\samptwo}\sbr{f^*_1\rbr{\frac{\zeta\rbr{s', a'} - \rbr{1 - \gamma}\init(s')\pi(a'|s') - \gamma \zeta\rbr{s, a}}{\alpha_Q\visitrb}}} - \alpha_\zeta \EE_{\visitrb}\sbr{f_2\rbr{\zeta}},
\end{align}
which will introduce bias.

\section{Undiscounted MDP}\label{appendix:undiscounted}
When $\gamma=1$, the value of a policy is defined as the average per-step reward:
\begin{equation}\label{eqn:avgstep-undiscount}
    \avgstep(\pi) \defeq \lim_{t_{\mathrm{stop}}\to\infty} \E\left[\frac{1}{t_{\mathrm{stop}}}\left.\sum_{t=0}^{t_{\mathrm{stop}}} \rew(s_t,a_t) ~\right| s_0\sim\init, \forall t, a_t\sim \pi(s_t), s_{t+1}\sim T(s_t, a_t)\right].
\end{equation}

The following theorem presents a formulation of $\rho(\pi)$ in the undiscounted case:
\begin{theorem}\label{thm:dual-succinct-undiscount}
Given a policy $\pi$ and a discounting factor $\gamma=1$, the value $\rho\rbr{\pi}$ defined in~\eqref{eqn:avgstep-undiscount} can be expressed by the following $d$-LP:
\begin{equation}\label{eq:dual-succinct-undiscount}
\textstyle\max_{d:S\times A\rightarrow \RR}\,\, \EE_{d}\sbr{\rew\rbr{s, a}},\quad \st,\quad d(s,a) = \bellmantnog d(s,a)\,\,\text{and}\,\,\sum_{s,a}\dvar(s,a) = 1.
\end{equation}
The corresponding primal LP under the undiscounted case is
\begin{equation}\label{eq:primal-succinct-undiscount}
\textstyle\min_{Q:S\times A\rightarrow \RR}\,\, \lambda,\quad \st,\quad   Q(s,a) = \rew(s,a) + \bellmannog Q(s,a) - \lambda.
\end{equation}
\end{theorem}
\begin{proof}
With the additional constraint $\sum_{s,a}\dvar(s,a) = 1$ in~\eqref{eq:dual-succinct-undiscount}, the Markov chain induced by $\pi$ is ergodic with a unique stationary distribution $d^*=d^\pi$, so the dual objective is still $\rho\rbr{\pi}$ by definition. Unlike in the discounted case, any optimal $Q^*$ with a constant offset would satisfy~\eqref{eq:primal-succinct-undiscount}, so the optimal solution to~\eqref{eq:primal-succinct-undiscount} is independent of $Q$.
\end{proof}

\section{Experiment Details}\label{appendix:exp}
\subsection{OPE tasks}
For all tasks, We use $\gamma=0.99$ in all experiments except for the ablation study of normalization constraint where $\gamma=0.995$ and $\gamma=1$ are also evaluated. We collect $400$ trajectories for each of the tasks, and the trajectory length for Grid, Reacher, and Cartpole are 100, 200, and 250 respectively for $\gamma < 1$, or $1000$ for $\gamma = 1$. 

\paragraph{Grid.} We use a $10\times10$ grid environment where an agent can move left/right/up/down. The observations are the $x, y$ coordinates of this agent's location. The reward of each step is defined as $\text{exp}(-0.2|x-9|-0.2|y-9|)$. The target policy is taken to be the optimal policy for this task (i.e., moving all the way right then all the way down) plus $0.1$ weight on uniform exploration. The behavior policies $\pi_1$ and $\pi_2$ are taken to be the optimal policy plus $0.7$ and $0.3$ weights on uniform exploration respectively.

\paragraph{Reacher.} We train a deterministic policy on the Reacher task from OpenAI Gym~\citep{brockman2016openai} until convergence, and define the target policy to be a Gaussian with the pre-trained policy as the mean and $0.1$ as the standard deviation. The behavior policies $\pi_1$ and $\pi_2$ have the same mean as the target policy but with $0.4$ and $0.2$ standard deviation respectively.

\paragraph{Cartpole.} We modify the Cartpole task from OpenAI Gym~\citep{brockman2016openai} to infinite horizon by changing the reward to $-1$ if the original task returns termination and $1$ otherwise. We train a deterministic policy on this task until convergence, and define the target policy to be the pre-trained policy (weight $0.7$) plus uniform random exploration (weight $0.3$). The behavior policies $\pi_1$ and $\pi_2$ are taken to be the pre-trained policy (weight $0.55$ and $0.65$) plus uniform random exploration (weight $0.45$ and $0.35$) respectively.

\subsection{Linear Parametrization Details}
To test estimation robustness to scaling and shifting of MDP rewards under linear parametrization,
we first determine the estimation upper bound by parametrizing the primal variable as a linear function of the one-hot encoding of the state-action input. Similarly, to determine the lower bound, we parametrize the dual variable as a linear function of the input. These linear parametrizations are implemented using feed-forward networks with two hidden-layers of $64$ neurons each and without non-linear activations. Only the output layer is trained using gradient descent; the rest layers are randomly initialized and fixed. The true estimates where both primal and dual variables are linear functions are verified to be between the lower and upper bounds.

\subsection{Neural Network Details}
For the neural network parametrization, we use feed-forward networks with two hidden-layers of $64$ neurons each and ReLU as the activation function. The networks are trained using the Adam optimizer ($\beta_1=0.99$, $\beta_2=0.999$) with batch size $2048$. The learning rate of each task and configuration is found via hyperparameter search, and is determined to be $0.00003$ for all configurations on Grid, $0.0001$ for all configurations on Reacher, and $0.0001$ and $0.00003$ for dual and primal regularization on Cartpole respectively.

\section{Additional Results}

\subsection{Comparison to unregularized Lagrangian}
We compare the best performing DICE estimator discovered in our unified framework to directly solving the Lagrangian without any regularization or redundant constraints, \ie, DR-MWQL as primal, MWL as dual, and their combination~\citep{uehara2019minimax}. Results are shown in~\figref{fig:est_reg_org}. We see that the BestDICE estimator outperforms the original primal, dual and Lagrangian both in terms of training stability and final estimation. This demonstrates that regularization and redundant constraints are crucial for optimization, justifying our motivation. 
\begin{figure}[h]
\centering
  \begin{subfigure}{1.\columnwidth}
    \includegraphics[width=1.\linewidth]{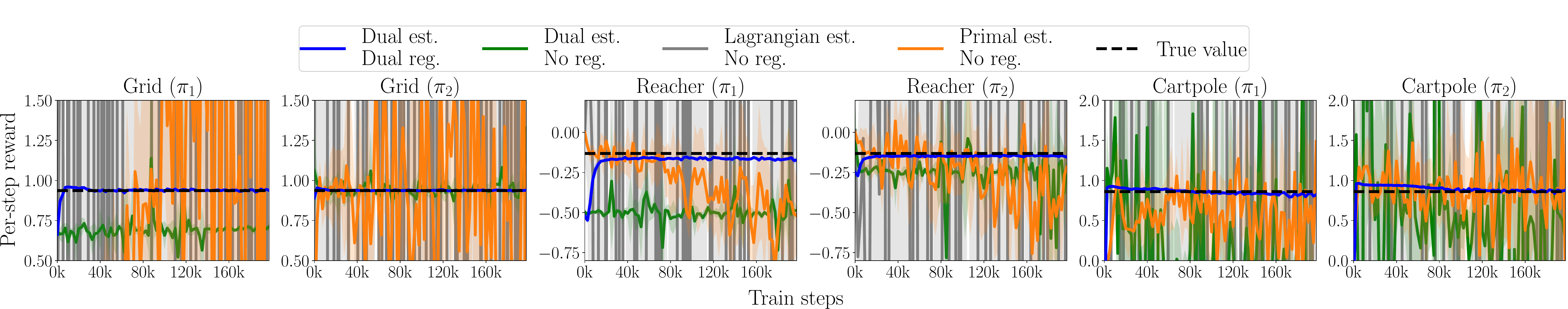}
  \end{subfigure}
  \caption{Primal (orange), dual (green), and Lagrangian (gray) estimates by solving the original Lagrangian without any regularization or redundant constraints, in comparison with the best DICE estimates (blue). }
\label{fig:est_reg_org}  
\end{figure}

\subsection{Primal Estimates with Target Networks}
We use target networks with double $Q$-learning~\citep{van2016deep} to improve the training stability of primal variables, and notice performance improvements in primal estimates on the Reacher task in particular. However, the primal estimates are still sensitive to scaling and shifting of MDP rewards, as shown in~\figref{fig:est_robust_target}.
\begin{figure}[h]
\centering
  \begin{subfigure}{1.\columnwidth}
    \includegraphics[width=1.\linewidth]{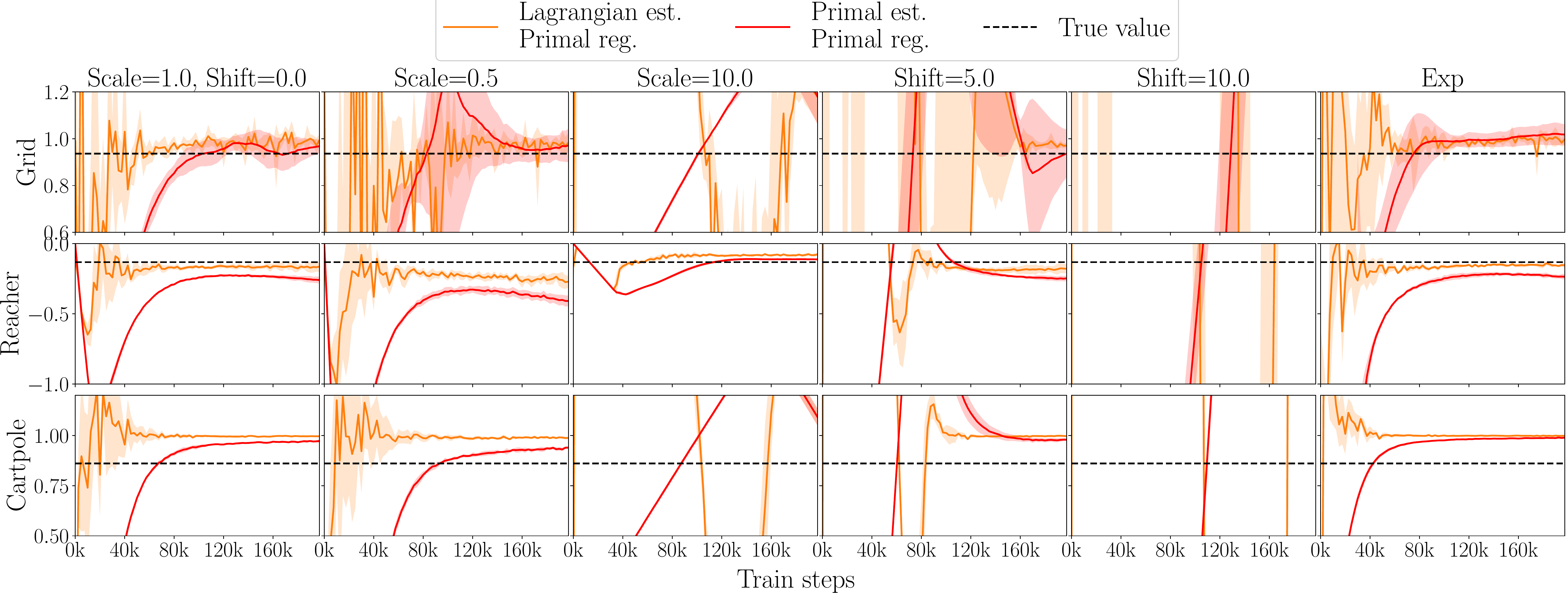}
  \end{subfigure}
  \caption{Primal (red) and Lagrangian (orange) estimates under the neural network parametrization with target networks to stabilize training when rewards are transformed during training. Estimations are transformed back and plotted on the original scale. Despite the performance improvements on Reacher compared to~\figref{fig:est_robust}, the primal and Lagrangian estimates are still sensitive to the reward values.}
\label{fig:est_robust_target}  
\end{figure}

\subsection{Additional Regularization Comparison}
In addition to the two behavior policies in the main text (i.e., $\pi_1$ and $\pi_2$), we show the effect of regularization using data collected from a third behavior policy ($\pi_3$). Similar conclusions from the main text still hold (i.e., dual regularizer is generally better; primal regularizer with reward results in biased estimates) as shown in~\figref{fig:reg_alph33}.
\begin{figure}[h]
\centering
\includegraphics[width=1.\linewidth]{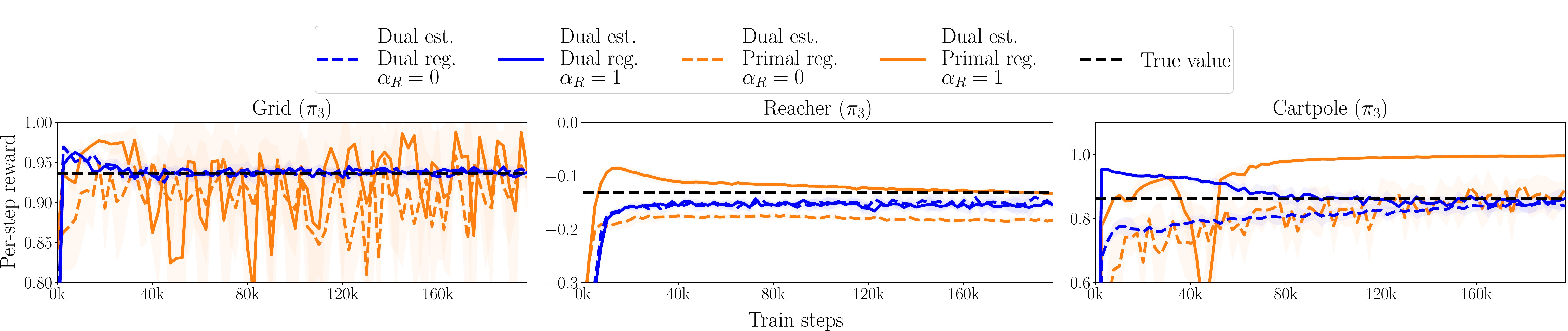}
\caption{Dual estimates when $\alpha_R = 0$ (dotted line) and $\alpha_R = 1$ (solid line) on data collected from a third behavior policy ($\pi_3$). Regularizing the dual variable (blue) is better than or similar to regularizing the primal variable (orange).}
\label{fig:reg_alph33}
\end{figure}

\subsection{Additional Ablation Study}
We also conduct additional ablation study on data collected from a third behavior policy ($\pi_3$). Results are shown in~\figref{fig:opt_alphas}. Again we see that the positivity constraint improves training stability as well as final estimates, and unconstrained primal form is more stable but can lead to biased estimates.
\label{sec:exp_opt_additional}
\begin{figure}[h]
  \begin{subfigure}{1.\columnwidth}
    \includegraphics[width=1.\linewidth]{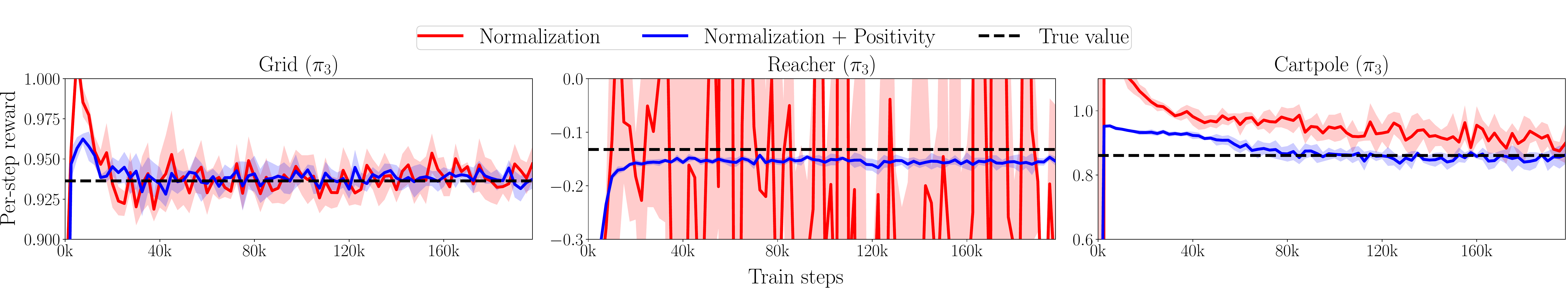}
  \end{subfigure}
  \begin{subfigure}{1.\columnwidth}
    \includegraphics[width=1.\linewidth]{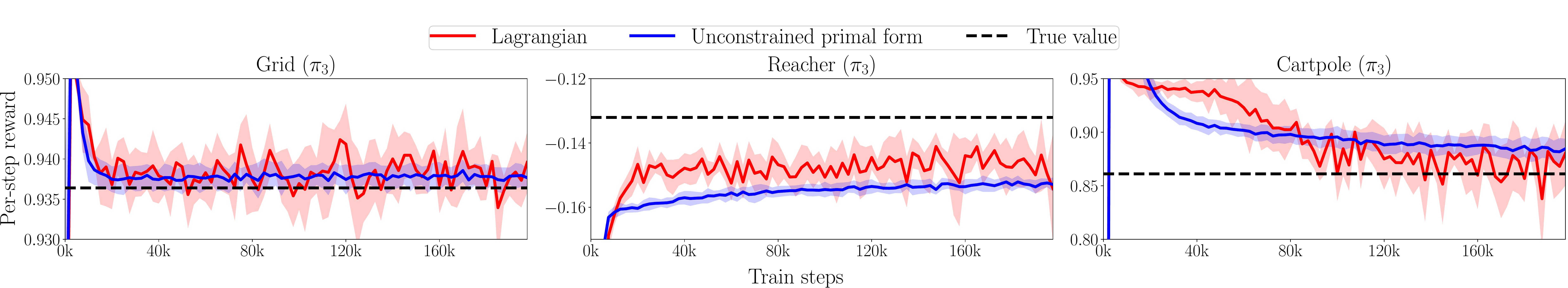}
  \end{subfigure}
  \caption{Apply positive constraint and unconstrained primal form on data collected from a third behavior policy ($\pi_3$). Positivity constraint (row 1) improves training stability. The unconstrained primal problem (row 2) is more stable but leads to biased estimates.}
\label{fig:opt_alphas}
\end{figure}

\end{appendix}

\end{document}